\documentclass{article}

\usepackage[main,final]{neurips_2026}
\usepackage[utf8]{inputenc}
\usepackage[T1]{fontenc}
\usepackage{microtype}
\usepackage{graphicx}
\usepackage{subcaption}
\usepackage{booktabs}
\usepackage{multirow}
\usepackage{float}
\usepackage{longtable}
\usepackage{algorithm}
\usepackage{algorithmic}
\usepackage{hyperref}
\hypersetup{hidelinks,draft}
\usepackage{url}
\usepackage{amsfonts}
\usepackage{nicefrac}
\usepackage{xcolor}
\usepackage{amsmath}
\usepackage{amssymb}
\usepackage{mathtools}
\usepackage{amsthm}
\usepackage[capitalize,noabbrev]{cleveref}
\usepackage{tabularx}
\usepackage{adjustbox}
\newcommand{\INPUT}[1]{\STATE \textbf{Input: } #1}

\theoremstyle{plain}
\newtheorem{theorem}{Theorem}[section]
\newtheorem{proposition}[theorem]{Proposition}
\newtheorem{lemma}[theorem]{Lemma}

\theoremstyle{definition}
\newtheorem{definition}[theorem]{Definition}
\newtheorem{assumption}[theorem]{Assumption}
\theoremstyle{remark}

\title{Residual Calibration via Local Feature-Space Refinement}

\author{Xihao Wang$^\text{a}$\thanks{Co-first authors and contributed equally to this work.}, \quad Chengyao Yu$^\text{b}$\footnotemark[1], \quad Ruixing Ming$^\text{a}$\thanks{Correspondence to: Ruixing Ming <rxming@mail.zjgsu.edu.cn>}\\
$^\text{a}$School of Statistics and Data Science,
Zhejiang Gongshang University;\\
$^\text{b}$Department of Statistics and Data Science,
Southern University of Science and Technology
}

\begin{document}

\maketitle
\begin{abstract}
Deep neural networks often struggle to provide reliable confidence estimates, making post-hoc calibration essential for reliable decision-making. Existing post-hoc methods largely rely on score-based transformations, which improve marginal calibration but often fail to capture spatially varying miscalibration in local regions of the learned feature space. To address this gap, we propose Local Adaptive Probability Estimation (LAPE), a post-hoc residual calibration framework. Specifically, LAPE treats the global calibrator output as a stable anchor and refines the individual prediction by leveraging non-parametric local evidence from neighboring calibration samples. As a plug-and-play refinement module, LAPE effectively utilizes the information about feature-space geometry and can reuse the calibration data that are used to construct the global calibrator. Furthermore, we extend LAPE to the setting with multi-view inputs through Cascade Feature Augmentation Fusion (CFAF). Extensive experiments demonstrate that the proposed methods can improve the calibration quality of various existing baselines and remain efficient and robust across different settings.
\end{abstract}

\section{Introduction}

Deep neural networks have exhibited remarkable accuracy on classification problems, but often fail to provide reliable confidence scores that reflect the ground truth correctness likelihoods~\citep{guo2017calibration}. This mismatch can appear as overconfidence or underconfidence, both of which are critical especially in high-stakes applications and downstream tasks relying prediction uncertainty~\citep{yu2026robust,yu2026anytime}.
For example, in automated clinical diagnostics, a reliable confidence score can indicate when a prediction should be deferred to experts for further examination, and in robust conformalized selection~\citep{yu2026robust}, a reliable confidence score can improve selection power.
Therefore, a network should provide reliable calibrated probabilities rather than directly applying the point predictions~\citep{amodei2016concrete, lakshminarayanan2017simple, jiang2018trust}.

Post-hoc calibration provides a cost-effective solution to rectify miscalibration without retraining. However, existing global parametric maps such as Temperature Scaling (TS)~\citep{guo2017calibration}, ensemble TS (ETS)~\citep{zhang2020mix}, and Dirichlet calibration~\citep{kull2019beyond} rely solely on scalar logits. Consequently, predictions with similar scores tend to obtain identical calibrations, ignoring the varying empirical reliability across distinct feature neighborhoods. To this end, local TS (LTS)~\citep{Ding_2021_ICCV} introduces spatial adaptivity, but its increased parametric capacity risks overfitting on limited calibration data~\citep{guo2017calibration, kull2019beyond, tomani2022parameterized}. Conversely, purely non-parametric local methods capture spatial patterns by adapting to neighborhood distributions but suffer from data inefficiency and discard informative global logits~\citep{zadrozny2001obtaining, zadrozny2002transforming, naeini2015obtaining, kumar2019verified}.
These limitations raise a foundational challenge:
\textit{how to develop a calibration framework that leverages local feature-space evidence to correct residual errors while strictly preserving global predictive stability?}

 To address the above challenge, we propose \textit{Local Adaptive Probability Estimation} (LAPE), a post-hoc residual calibration framework that effectively utilizes local feature-level evidence to further enhance existing calibration methods. Specifically, LAPE formulates calibration as a \textit{residual refinement} problem. It takes the output of any existing post-hoc calibrator as a stable anchor (also called global anchor), and then applies a local correction to it only when the nearby calibration samples support it.
To tackle the overfitting issue caused by reusing the calibration set to fit both a global anchor and residual refinement, we introduce a $V$-fold cross-fitting strategy. It generates out-of-sample anchor scores in the calibration set and can be more effective than the standard data-splitting strategy.
Finally, we develop \textit{Cascade Feature Augmentation Fusion} (CFAF), which extends LAPE to tackle multi-view scenarios~\citep{DBLP:journals/corr/abs-1304-5634, baltruvsaitis2018multimodal}.

In theory, we establish the asymptotic consistency of the proposed LAPE. Empirically, we extensively evaluate LAPE across diverse benchmarks that encompass tabular and image-type data. Results show that LAPE helps to enhance the calibration performance of various baselines, such as TS~\citep{guo2017calibration}, Feature Clipping (FC)~\citep{tao2025fc}, and Semantic-Aware Grouping (SAG)~\citep{yang2023sag}. Extensive experiments in Appendix~\ref{app:supplementary_results} further demonstrate the efficiency, robustness, and effectiveness of LAPE.

We summarize our contributions as follows:
\begin{itemize}
    \item[1.] We propose LAPE, a residual calibration framework that enhances existing post-hoc calibration methods by integrating feature-level evidence. Importantly, LAPE can reuse the calibration dataset for building the global calibrator and residual refinement.

    \item[2.] We extend LAPE to multi-view settings through CFAF, which combines calibrated view-level predictions with local refinement on a rectified joint feature space.

    \item[3.]  We establish asymptotic consistency of LAPE, providing a consistency analysis for feature-space residual calibration of the full multiclass probability vector.

    \item[4.] We conduct comprehensive evaluations across various datasets and architectures, demonstrating that LAPE can improve the performance of existing post-hoc calibration methods.
\end{itemize}

\section{Preliminary}
\label{sec:related_work}

\subsection{Problem Setup}

Let $X\in \mathcal{X}$ denote the feature variable and $Y\in\mathcal{Y}=\{1,\dots, K\}$ denote the label variable.
Assume access to a calibration dataset $\{(x_i,y_i)\}_{i=1}^{n}$.
Let $\hat{\pi}_0:\mathcal{X}\to \Delta^{K-1}$ be a pre-trained soft-classifier for the $K$-class classification, where $\Delta^{K-1}$ represents the $(K-1)$-dimensional probability simplex.
Let $p(k;x)=\mathbb{P}(Y=k~|~X=x)$.
For any $x\in\mathcal{X}$, we write $\hat{\pi}_0(x)=(\hat{\pi}_{0}(1;x),\dots,\hat{\pi}_{0}(K;x))$, where $\hat{\pi}_0(k;x)$ represents the base estimate of $p(k;x)$.
Denote the feature embedding of $x$ by $z:=\phi(x)$ and write $z_i=\phi(x_i)$, $i\in[n]:=\{1,\dots,n\}$.
The initial goal of LAPE is to construct a calibrated estimator of $p(k;x)$, denoted by $\hat{p}(k;x)$, by leveraging $\hat{\pi}_0(k;x)$, $\mathcal{D}_{\text{cal}}$, and $\{z_i\}_{i=1}^n$. 

Furthermore, let $f_{\text{cal}}\left(\cdot;\hat{\pi}_{\text{0}}, \mathcal{D}\right):\mathcal{X}\to \Delta^{K-1}$ be \textit{any} existing post-hoc calibration algorithm, which calibrates the initial soft-classifier $\hat{\pi}_0$ based on a calibration dataset $\mathcal{D}\subseteq \mathcal{D}_{\text{cal}}$. For example, the mapping $f_{\text{cal}}$ can be defined by calibration algorithms such as VS and TS. We aim to utilize LAPE to further improve these calibration methods.

\subsection{Motivation}

A learned representation can be discriminative for classification while still leaving calibration errors that vary across different feature-space neighborhoods. For example, global post-hoc calibrators typically learn a fixed transformation from logits or confidence scores to calibrated probabilities.
As a result, two samples with similar scores receive similar corrections, even if their nearby calibration samples exhibit different empirical label frequencies~\citep{tomani2022parameterized,pmlr-v180-luo22a,Ding_2021_ICCV}.
Such ``\textit{same-score, different-reliability}'' implies a possible route beyond purely score-based calibration: retain the global calibrator as a \textit{stable anchor} and use nearby held-out samples as \textit{local evidence} for residual mismatch, when the feature geometry makes such neighborhoods informative.

The key insight of LAPE is to use local frequencies in calibration samples as residual evidence for probability calibration. Figure~\ref{fig:lape_concept} illustrates this mechanism in a controlled two-dimensional simulation, where training and calibration samples are generated from the same distribution, and only the training samples are used to fit the classifier. The figure highlights two observations as follows.

\begin{figure*}[t]
    \centering
    \begin{minipage}[t]{0.48\textwidth}
        \centering
        \includegraphics[width=\textwidth]{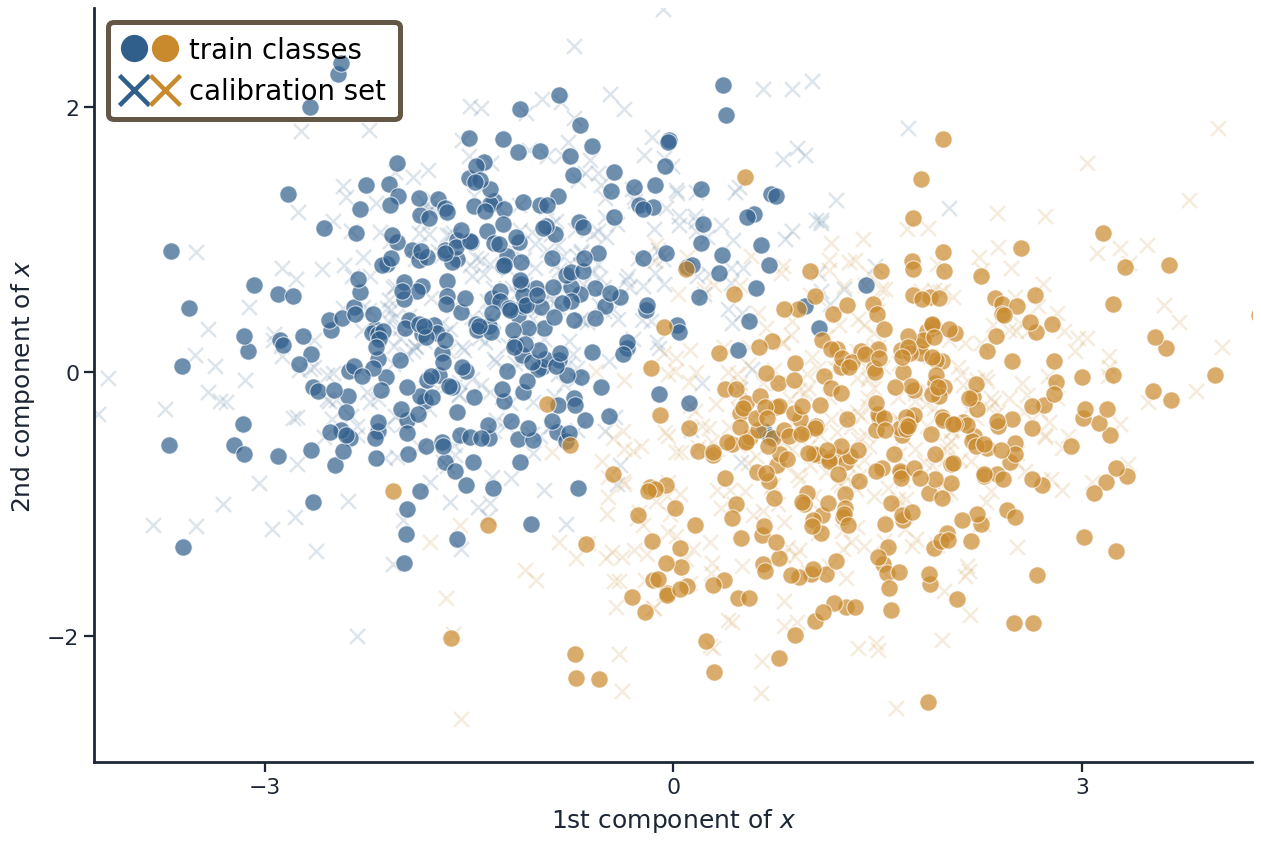}
        
        \vspace{0.05cm}
        \centerline{\small (a) Raw feature space}
    \end{minipage}
    \hfill
    \begin{minipage}[t]{0.48\textwidth}
        \centering
        \includegraphics[width=\textwidth]{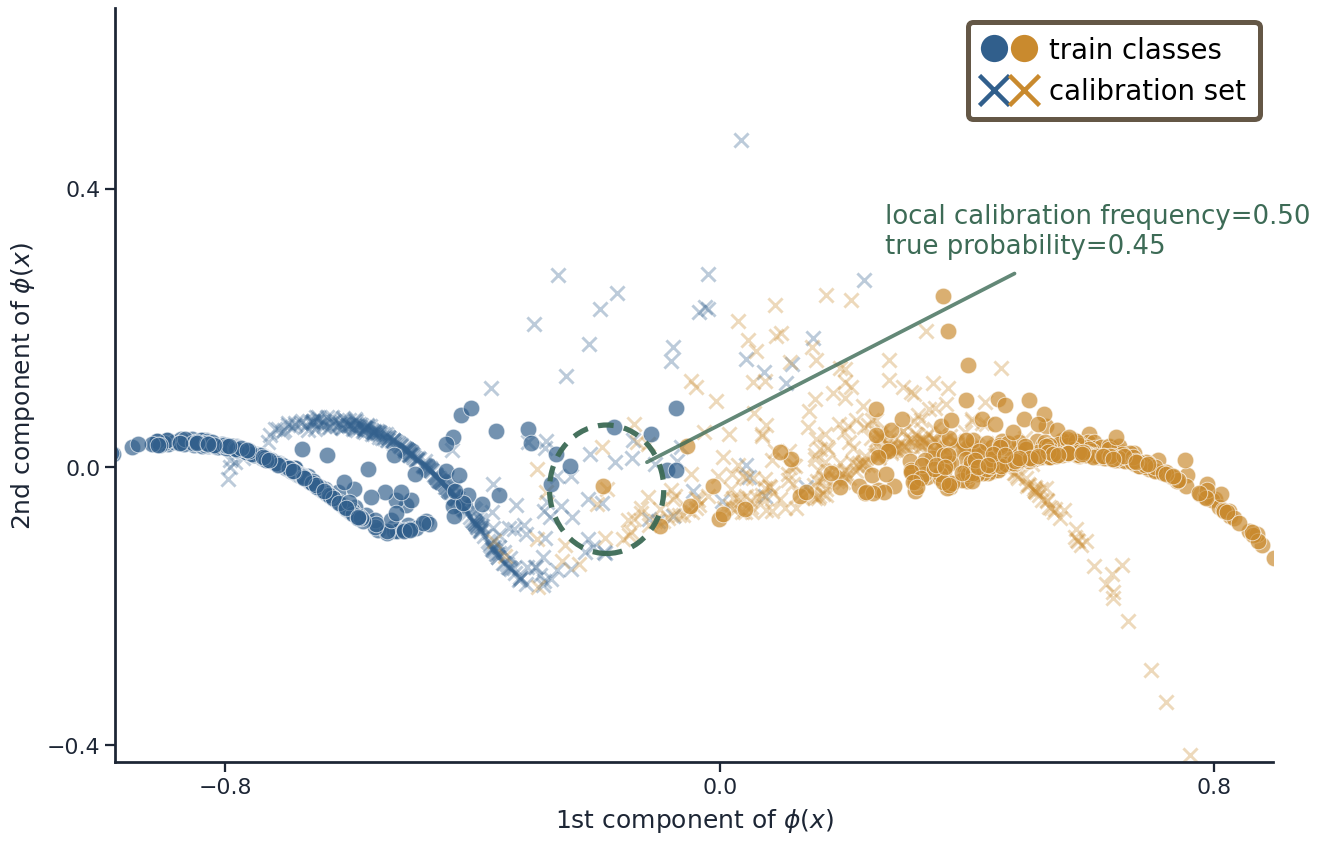}
        
        \vspace{0.05cm}
        \centerline{\small (b) Learned feature space}
    \end{minipage}
    \vspace{0.1cm}
    \caption{
    \textbf{Representation separation can not guarantee uniform calibration support.} (a) Raw features of training and calibration samples. (b) Learned features separate the training samples clearly, but calibration samples can remain mixed near induced boundary regions. Different colors represent different classes.
    The highlighted neighborhood illustrates the local evidence used by LAPE.}
    \label{fig:lape_concept}
\end{figure*}

\paragraph{Observation 1.} After representation learning, the training samples become more clearly separated, whereas held-out calibration samples can still remain locally mixed near induced boundary regions.

\paragraph{Observation 2.}
For held-out points with class-1 softmax scores in $[0.35,0.65]$, local frequencies from 32 nearest neighbors in the full learned feature space better approximate oracle probabilities than softmax predictions (MAE: 0.126 vs.\ 0.156).

Observation 1 reflects that the classifier is optimized to separate the observed training dataset but can fail to separate held-out samples near classification boundaries.
Observation 2 demonstrates that nearby held-out calibration samples can provide empirical evidence about the local class composition that is not directly used by a purely score-based global calibrator.
LAPE therefore uses this local frequency as a residual correction to a global calibrated anchor.

{
\paragraph{Evidence beyond two dimensions.}

We further examine the above observations in a controlled simulation with 100-dimensional inputs. 
We generate i.i.d. training, calibration, and test sets, each containing 2,000 samples.
The classifier is obtained on the training dataset.
For a split \(A\in\{\mathrm{train},\mathrm{cal}\}\) and a representation 
\(h\in\{\mathrm{raw},\mathrm{learned}\}\), let \(\mu_{A,c}^{h}\) denote the centroid of class \(c\):
\[
\mu_{A,c}^{h}
=
\frac{1}{|\mathcal I_{A,c}|}
\sum_{i\in\mathcal I_{A,c}} h(x_i),
\qquad
\mathcal I_{A,c}
=
\{i:x_i\in A,\ y_i=c\}.
\]
We use $S_A^{h} =\left\| \mu_{A,1}^{h}-\mu_{A,0}^{h}\right\|_2$ as a class-separation diagnostic, after standardizing each representation using statistics from the training split.

Results are averaged over ten random seeds. 
In the raw input space, the training and calibration sets have nearly identical class separation (\(2.510\pm0.176\) versus \(2.519\pm0.162\)). 
After representation learning, the separation is moderately larger for the training samples than for the held-out calibration samples (\(3.663\pm0.343\) versus \(2.981\pm0.234\)), showing that a modest train--calibration discrepancy can remain even under i.i.d. sampling.

We then evaluate probability estimation on an independent test set. 
The boundary region is specified before model fitting as
\(0.35\leq P(Y=1\mid X)\leq0.65\).
For each test point, the local estimate is computed from the Laplace-smoothed label frequency of its 50 nearest calibration samples in the learned representation. 
Within this predefined region, the mean absolute error to the oracle conditional probability is
\(0.1067\pm0.0098\) for the softmax estimate and
\(0.0969\pm0.0055\) for the local-frequency estimate.
Over the entire test samples which includes points that fall outside the boundary region, however, softmax remains more accurate than the local estimate
(\(0.0941\pm0.0054\) versus \(0.1058\pm0.0042\)). 
These results support using local calibration frequencies as a residual correction in selected regions, rather than as a replacement for the global prediction.

\section{Local Adaptive Probability Estimation} \label{sec:methodology}

This section presents LAPE in two parts. We first introduce the general LAPE method, which constructs a non-parametric residual correction for a pre-trained soft-classifier \(\hat{\pi}_0\). We then describe how to apply LAPE to existing post-hoc calibrators by treating their calibrated outputs as \(\hat{\pi}_0\) under the shared calibration dataset \(\mathcal{D}_{\mathrm{cal}}\).

\subsection{General Methodology}\label{general_method}

We introduce the general method of LAPE in three steps. We first describe how to identify the neighbors of a given test point. Then, we introduce the hybrid refinement, which fuses the global anchor with this local evidence to produce the final calibrated probability. Finally, we describe how to employ the adaptive optimization to solve for the correction parameters that minimize residual errors.

\paragraph{Step 1: Local evidence retrieval.}
Given a new test sample $x$ and its feature embedding vector $z$, the first step of LAPE is to select samples from $\mathcal{D}_{\text{cal}}$ that are geometrically close to $x$ in the feature space. Specifically, let $d_i=\|z_i-z\|_2$, $i\in [n]$. For any fixed $m\in[n]$, the set of $m$ nearest neighbors $\mathcal{N}_{m}(x)$ is defined by
\begin{equation*}
\mathcal{N}_{m}(x) := \left\{ i\in[n] : \|z - z_i\|_2 \le d_{(m)}\right\},
\end{equation*}
where $d_{(m)}$ denotes the $m$-th smallest value of $\{d_i\}_{i=1}^n$. By leveraging the local calibration dataset $\mathcal{N}_m(x)$, LAPE then computes the local empirical distribution as
\begin{equation}\label{eq:p_emp}
\tilde{p}_{0}(k;x, m) := \frac{\sum_{i\in\mathcal{N}_{m}(x)}\mathbb{I}(y_{i}=k) + 1/K}{m+1},\quad k\in[K],
\end{equation}
where $1/K$ is a smoothing term to avoid zero probabilities.

{
\paragraph{Step 2: Hybrid refinement.} Let $\pi^{\text{anchor}}(k;x)$ be the anchor estimate of $p(k;x)$, $k\in[K]$. We focus on the case that 
\[
\pi^{\text{anchor}}(k;x)=\hat{\pi}_0(k;x).
\]
The idea of hybrid refinement is that, if the base prediction $\hat{\pi}_0(\cdot)$ is already reliable in many regions, then it is unnecessary to adopt a large local correction. For sparse regions where $\hat{\pi}_0(\cdot)$ performs not satisfactorily, local frequencies may also be weakly supported, which means that over-reliance on the local evidence can further increase variance. 
Therefore, LAPE treats $\tilde{p}_0(k;x,m)$ as a delicate \textit{residual correction signal} rather than a replacement for the global anchor. Specifically, given a neighborhood size $m\in[n]$ and a weight $\alpha \in [0,1]$, the calibrated probability of $\hat{\pi}_0(k;x)$ is given by
\begin{equation}\label{eq:final_convex}
\hat{p}(k;x,m,\alpha) := (1 - \alpha) \cdot \pi^{\text{anchor}}(k;x) + \alpha \cdot \tilde{p}_{0}(k;x, m),\quad k\in[K].
\end{equation}
In what follows, we describe how to choose $m$ and $\alpha$ adaptively.
}

{
\paragraph{Step 3: Adaptive hyperparameter selection.}
LAPE selects the hyperparameter vector $(m,\alpha)$ by leveraging weighted leave-one-out (LOO) analysis. Specifically, for any fixed $i\in[n]$, we denote the LOO probability estimate of $x_i$ by
\begin{equation}\label{eq:lape_mixed_prob}
p^{(i)}(y_i;x_i,m,\alpha) := (1 - \alpha)\pi^{\text{anchor}}(y_i;x_i)
+ \alpha \tilde{p}_{-i}(y_i;x_i,m).
\end{equation}
where $\tilde{p}_{-i}(k;x_i,m)$ is the LOO local empirical distribution that
\begin{equation}\label{eq:loo_empirical}
\tilde{p}_{-i}(k;x_i,m) := \frac{\sum_{j \in \mathcal{N}_m(x) \setminus \{i\}} \mathbb{I}(y_j = k) + 1/K}{m},\quad k\in[K].
\end{equation}
Each calibration sample is weighted by its distance to the feature embedding of the test point $x$, i.e.,
\begin{equation}\label{eq:kernel_weight}
w_i^{(m)} =
\frac{\exp\left(-\|z-z_i\|_2^2/\sigma^2\right)}
{\sum_{j\in\mathcal{N}_m(x)}\exp\left(-\|z-z_j\|_2^2/\sigma^2\right)},
\end{equation}
where $\sigma=d_{(\lfloor \sqrt{n}/2 \rfloor)}$ is the bandwidth. We optimize $m$ over the grid $\mathcal{M}=\{M_{\min},\dots,M_{\max}\}\subseteq[n]$, where $M_{\text{min}}$ and $M_{\text{max}}$ have the same order $O(\sqrt{n})$. In practice, we simply set $M_{\text{min}}=\lfloor \sqrt{n}/10 \rfloor$ and $M_{\text{max}}=\lfloor \sqrt{n}/2 \rfloor$. Finally, the hyperparameter vector $(m^*,\alpha^*)$ is adaptively determined by minimizing the density-weighted negative log-likelihood:
\begin{equation}\label{eq:loo_optimization}
(m^*, \alpha^*) = \operatorname*{arg\,min}_{m \in \mathcal{M}, \alpha \in [0,1]} \sum_{i\in\mathcal{N}_m(x)} -w_i^{(m)} \cdot \log \left( p^{(i)}(y_i; x_i, m, \alpha) \right).
\end{equation}
Correspondingly, the final calibrated estimate $\hat{p}(x):=(\hat{p}(1;x),\dots,\hat{p}(K;x))$ satisfies
\begin{equation}\label{eq:final}
\hat{p}(k;x) = (1 - \alpha^*) \cdot \pi^{\text{anchor}}(k;x) + \alpha^* \cdot \tilde{p}_{0}(k;x, m^*),\quad k\in[K].
\end{equation}
The proposed method is summarized in Algorithm~\ref{alg:lape}.

\begin{algorithm}[t]
\caption{Local Adaptive Probability Estimation}
\label{alg:lape}
\begin{algorithmic}[1]
\INPUT Test sample \(x\), calibration set 
\(\mathcal D_{\mathrm{cal}}=\{(x_i,y_i)\}_{i=1}^{n}\), feature map \(\phi\),
anchor scores \(\pi^{\mathrm{anchor}}(x)\) and
\(\{\pi^{\mathrm{anchor}}(x_i)\}_{i=1}^{n}\), candidate grids
\(\mathcal M\) and \(\mathcal A\).
\STATE {\bfseries Output:} Calibrated probability vector \(\hat p(x)\).

\STATE Compute \(z=\phi(x)\) and \(z_i=\phi(x_i)\), \(i\in[n]\).
\STATE Compute the neighbor sets \(\mathcal N_m(x)\) for all \(m\in\mathcal M\).
\STATE Set $\sigma=d_{(\lfloor\sqrt n/2\rfloor)}$, the distance from $z$ to its
$\lfloor\sqrt n/2\rfloor$-th nearest calibration embedding.

\FOR{\(m\in\mathcal M\) and \(\alpha\in\mathcal A\)}
    \STATE Initialize \(L(m,\alpha)\leftarrow 0\).
    \FOR{\(i\in\mathcal N_m(x)\)}
        \STATE Compute the local weight $w_i^{(m)}$ as in \eqref{eq:kernel_weight}.
        \STATE Compute the leave-one-out local frequency $\tilde p_{-i}(k;x_i,m)$ as in \eqref{eq:loo_empirical}.
        \STATE Compute $\hat p^{(i)}(y_i;x_i,m,\alpha)$ as in \eqref{eq:lape_mixed_prob}.
        \STATE Update 
        \[
        L(m,\alpha)\leftarrow L(m,\alpha) -w_i^{(m)}(x) \log \hat p^{(i)}(y_i;x_i,m,\alpha).
        \]
    \ENDFOR
\ENDFOR

\STATE Select $(m^*,\alpha^*) = \arg\min_{m\in\mathcal M,\alpha\in\mathcal A} L(m,\alpha)$.
\STATE Compute $\hat{p}(k;x)$ as in \eqref{eq:final}, $k\in[K]$.

\STATE \textbf{Return} $\hat{p}(x)=(\hat{p}(1;x),\dots, \hat{p}(K;x))$.
\end{algorithmic}
\end{algorithm}

}

\subsection{Applying LAPE to Calibrators with Shared Calibration Dataset}\label{sec:32}
{
By \eqref{eq:final}, the quality of the global anchor $\hat{\pi}_{0}(x)$ is important to the calibrated $\hat{p}(x)$. Directly taking the softmax probability vector of a neural network as $\hat{\pi}_0(x)$ is concerning, as it usually fails to provide a reasonable estimate of $(p(1;x),\dots,p(K;x))$. An appealing solution is to apply LAPE to \textit{already calibrated estimates}. For example, we can take the result of VS as $\hat{\pi}_0(x)$. The remaining concern is that existing confidence calibration methods also need a calibration dataset, and reuse the same set $\mathcal{D}_{\text{cal}}$ for fitting $\hat{\pi}_0(k;x)$ and constructing $\hat{p}(k;x)$ will result in calibration-set overfitting.
A direct idea is to split $\mathcal{D}_{\text{cal}}$ into two disjoint parts, but this strategy often comes at the cost of sacrificing the utilization of the calibration set. To tackle this, we introduce a $V$-fold cross-fitting strategy as follows.

Specifically, we first partition $\mathcal{D}_{\text{cal}}$ into $V$ disjoint subsets $\{\mathcal{D}_{\text{cal}}^{(v)}$, $v=1,\dots,V\}$. 
For each fold $\mathcal{D}_{\text{cal}}^{(v)}$, we first calibrate $\hat{\pi}_0$ on the set $\mathcal{D}_{\text{cal}} /\ \mathcal{D}_{\text{cal}}^{(v)}$ to obtain a \textit{calibrated soft-classifier} $f_{\text{cal}}:\mathcal{X}\to  \Delta^{K-1}$. We then apply $f_{\text{cal}}$ to the hold-out calibration dataset $\mathcal{D}_{\text{cal}}^{(v)}$. That is, we take $\pi^{\text{anchor}}(x_i)=(\pi^{\text{anchor}}(1;x_i),\dots,\pi^{\text{anchor}}(K;x_i))$ as
\begin{equation*}
\pi^{\text{anchor}}(x_i) = f_{\text{cal}}\left(x_i;\hat{\pi}_{\text{0}}, \mathcal{D}_{\text{cal}} \setminus \mathcal{D}_{\text{cal}}^{(v)}\right), \quad \forall (x_i,y_i) \in \mathcal{D}_{\text{cal}}^{(v)},\quad v=1,\dots,V.
\end{equation*}
This ensures that the prediction function for the point $x_i$ is independent of $(x_i,y_i)$, $i\in[n]$. For the test point $x$, due to its independence from $\mathcal{D}_{\text{cal}}$, the calibrated function is fitted on the whole set $\mathcal{D}_{\text{cal}}$:
\begin{equation*}
\pi^{\text{anchor}}(x)= f_{\text{cal}}\left(x;\hat{\pi}_{\text{0}}, \mathcal{D}_{\text{cal}}\right).
\end{equation*}
The remaining steps are identical to subsection~\ref{general_method}.
}

\section{Extension to Multi-View Scenarios}\label{sec:cfaf}

{ 
For many applications, object may be described by several feature groups or multiple modalities; we refer to these as multi-view scenarios.
It is common in post-hoc deployment that the view-specific soft-classifiers have already been trained, and retraining an end-to-end interaction model may be costly or infeasible. Given a shared held-out calibration dataset $\mathcal{D}_{\text{cal}}$, a naive calibration strategy is late fusion, which calibrates each view separately and then averages them to obtain a final calibrated probability \citep{snoek2005early,wolpert1992stacked,baltruvsaitis2018multimodal}. However, this strategy operates on predictors that are trained separately for each view, which can miss cross-view interactions.
The Cascade Feature Augmentation Fusion (CFAF) is designed to tackle this problem.
}

{
Let each input be observed through \(R\) views:
\(x=(x^{(1)},\ldots,x^{(R)})\). For each view \(r\in[R]\), let \(\hat\pi_{0}^{(r)}:\mathcal X^{(r)}\to\Delta^{K-1}\) be the frozen view-specific
soft-classifier and let \(z^{(r)}=\phi^{(r)}(x^{(r)})\) be its feature representation.
For the test data $x$, we fit a post-hoc calibrator for each view and obtain calibrated view-level predictions
\[
\hat q^{(r)}(x)= f_{\mathrm{cal}}^{(r)}\bigl(x^{(r)};\hat\pi_{0}^{(r)},\mathcal{D}_{\mathrm{cal}}\bigr),\quad r\in[R].
\]
Then we apply the \(V\)-fold cross-fitting strategy in Section~\ref{sec:32} to compute \(\hat q_{i}^{(r),\mathrm{oof}}\), which represents the out-of-fold calibrated prediction of
view \(r\) for \(x_i\). We compute the fusion weights by
\begin{equation}\label{fusion_w}
w^*=\arg\min_{w\in\Delta_R}-\sum_{i=1}^{n}\log\left(
\sum_{r=1}^{R} w_r \hat q_{i}^{(r),\mathrm{oof}}(y_i)
\right),
\qquad
\Delta_R=\left\{w:w_r\ge 0,\ \sum_{r=1}^{R}w_r=1\right\}.
\end{equation}
Therefore, the resulting CFAF anchor for a test sample is
\begin{equation}\label{anchor_cfaf}
\pi_{\mathrm{CFAF}}^{\mathrm{anchor}}(x)=\sum_{r=1}^{R} w_r^* \hat q^{(r)}(x).
\end{equation}
CFAF then applies LAPE in a \textit{joint feature space} rather than in any single view.
Specifically, for $i\in[n]$, define the concatenated feature representation by $z_i^{\mathrm{joint}}=[z_i^{(1)}\oplus\cdots\oplus z_i^{(R)}]$. For the test point $x$, write $z^{\mathrm{joint}}
=[z^{(1)}\oplus\cdots\oplus z^{(R)}]$. We apply a rectification map fitted on the calibration features to obtain $\tilde z_i=\Psi(z_i^{\mathrm{joint}};\mathcal D_{\mathrm{cal}})$, $i\in[n]$, where \(\Psi\) denotes standardization, optionally followed by whitening. Write $\tilde z=\Psi(z^{\mathrm{joint}};\mathcal D_{\mathrm{cal}})$.
Using
\(\{\tilde z_i,y_i\}_{i=1}^n\) for neighbor retrieval, CFAF computes the same local
empirical distribution as LAPE:
\[
\tilde p_0(k;x,m) = \frac{\sum_{i\in\mathcal N_m(\tilde z)}\mathbb I(y_i=k)+1/K}{m+1},
\quad k\in[K].
\]
With \((m^*,\alpha^*)\) selected as in \eqref{eq:loo_optimization}, the final CFAF prediction is
\begin{equation}\label{fusion_cfaf}
\hat p(k;x) = (1-\alpha^*)\pi_{\mathrm{CFAF}}^{\mathrm{anchor}}(k;x) + \alpha^*\tilde p_0(k;x,m^*), \quad k\in[K].
\end{equation}
The CFAF method is summarized in Algorithm~\ref{alg:cfaf}.
}

\begin{algorithm}[t]
\caption{LAPE with Cascade Feature Augmentation Fusion}
\label{alg:cfaf}
\begin{algorithmic}[1]
\INPUT Test sample \(x\), soft-classifiers \(\{\hat\pi_0^{(r)},\phi^{(r)}\}_{r=1}^{R}\), calibration set
\(\mathcal D_{\mathrm{cal}}\), search grid \(\mathcal A\).
\STATE {\bfseries Output:} Calibrated probability vector \(\hat p(x)\).
\STATE Compute out-of-fold calibrated view predictions
\(\{\hat q_i^{(r),\mathrm{oof}}\}\) on \(\mathcal D_{\mathrm{cal}}\).
\STATE Compute fusion weights $w^*$ by \eqref{fusion_w}.
\STATE Fit view-level calibrators on the full \(\mathcal D_{\mathrm{cal}}\) and compute $\pi_{\mathrm{CFAF}}^{\mathrm{anchor}}(x)$ by \eqref{anchor_cfaf}.
\STATE Construct joint features $z_i^{\mathrm{joint}}=[z_i^{(1)}\oplus\cdots\oplus z_i^{(R)}]$ for $i\in[n]$ and $z^{\mathrm{joint}}=[z^{(1)}\oplus\cdots\oplus z^{(R)}]$.
\STATE Compute $\tilde z_i=\Psi(z_i^{\mathrm{joint}};\mathcal D_{\mathrm{cal}})$, $i\in[n]$, and $\tilde z=\Psi(z^{\mathrm{joint}};\mathcal D_{\mathrm{cal}})$.
\STATE Apply LAPE to \(\pi_{\mathrm{CFAF}}^{\mathrm{anchor}}(x)\) and \(\{(\tilde z_i,y_i)\}_{i=1}^{n}\) as in \eqref{fusion_cfaf} to compute $\hat{p}(k;x)$, $k\in[K]$.
\STATE\textbf{Return} \(\hat p(x)\).
\end{algorithmic}
\end{algorithm}

\section{Theoretical Results}
\label{sec:theory}

{
In this section, we first establish the asymptotic consistency of LAPE; detailed proofs are deferred to Appendix \ref{proofs}. Then we provide a geometric interpretation of its finite-sample behavior.
}

{
We regard consistency as an important statistical property needed by a post-hoc calibration framework. That is, as the size of calibration dataset $n$ tends to infinity, the estimated probability vector should converge in probability to its true conditional distribution. Specifically, let
\[
p(x):=(p(1;x),\dots,p(K;x))=(\mathbb{P}(Y=1~|~X=x),\dots,\mathbb{P}(Y=K~|~X=x)).
\]
Rewrite the calibrated estimate vector $\hat{p}(x)$ under $\mathcal{D}_{\text{cal}}$ by $\hat{p}_n(x)$. Theorem~\ref{thm:consistency_lape} demonstrates that, under regular assumptions
(Assumptions~\ref{ass:reg_prob}--\ref{ass:candidate_grid};
see Appendix~\ref{assump}), the proposed estimator $\hat p_n(x)$ satisfies
$\hat p_n(x)\to p(x)$ in probability as $n\to\infty$.

\begin{theorem}[Consistency of LAPE]\label{thm:consistency_lape}
Let $x \in \mathcal{X}$ be a test sample with unobserved label $y$.
Suppose that $\mathcal{D}_{\mathrm{cal}}\cup\{(x,y)\}$ are i.i.d.
Let $\hat p_n(x)$ denote the LAPE estimator defined by \eqref{eq:final} when
constructed from a calibration set of size $n$.
Under Assumptions \ref{ass:reg_prob}--\ref{ass:candidate_grid}, the estimator
$\hat p_n(x)$ is asymptotically consistent. That is, for any fixed
$\varepsilon>0$,
\[
\lim_{n \to \infty}
\mathbb{P}\left(\|\hat p_n(x)-p(x)\|_2>\varepsilon\right)=0,
\]
where $\|\cdot\|_2$ denotes the Euclidean norm.
\end{theorem}
}

{

\paragraph{Geometric Interpretation.}
The term $\tilde{p}_0(k;x,m)$ given by \eqref{eq:p_emp} can be regarded as a smoothed local frequency that estimates the local conditional class probability around the test point $x$. The bias is therefore controlled by how much the conditional probability changes inside the retrieved neighborhood. Specifically, let $B_{\sigma}(z)$ be the neighborhood of radius $\sigma(x)$ around $z=\phi(x)$.
Ignoring the $O(1/m)$ contribution from Laplace smoothing, the geometric component of the local estimation bias satisfies
\begin{equation*}
\mathrm{Bias}(x) \le
\sup_{\phi(x')\in B_{\sigma}(z)}
\|p(x')-p(x)\|_2
\le L_p\cdot\sigma(x),
\end{equation*}
where $\mathrm{Bias}(x)$ denotes the unsmoothed geometric bias and $L_p$ is a local Lipschitz constant of the conditional probability map in $B_{\sigma}(z)$.
Thus, local refinement is reliable when nearby points have similar conditional probabilities. When the neighborhood mixes regions with different class probabilities, the local frequency can be biased; LAPE can then shrink the neighborhood or reduce the interpolation weight, relying more on the anchor prediction. Table~\ref{tab:oracle_smoothness_diagnostic} in Appendix \ref{appendix:Oracle_Neighborhood} supports this interpretation: LAPE is closer to the true probability in locally smooth neighborhoods, whereas the oracle gap increases when neighborhood variation is large.

}

\section{Experimental Results}\label{sec:experiments}

In this section, we present the experimental results of LAPE and other baselines on diverse datasets and neural network architectures.
We aim to demonstrate that the performance of confidence calibration of existing baselines can be enhanced by leveraging LAPE.
Codes can be found at \url{https://github.com/AltonWongsource/residual-calibration}.

\subsection{Experimental Setup}
\label{sec:protocol_fairness}
We introduce the main information about the experiments; detailed settings and additional results such as ablation studies are presented in Appendix \ref{app:experimental_details} and \ref{app:supplementary_results}, respectively.

\paragraph{Datasets.}
For the main result in Figure~\ref{fig:main_method_suite}, we evaluate our framework on four datasets that encompass both tabular and image categories. For tabular data, we consider Obesity Risk~\citep{palechor2019dataset}. For image data, we consider CIFAR-100~\citep{Krizhevsky2009LearningML} and two medical-image datasets: BreastMNIST~\citep{medmnistv2} and MHIST~\citep{wei2021petri}. We also evaluate the proposed method on many other datasets; see Appendix~\ref{app:datasets} for full descriptions.

\paragraph{Model Architectures}
For the tabular dataset, we use an MLP backbone. For medical transfer learning datasets, we use ImageNet-pretrained ResNet18 or DenseNet121 backbones. For CIFAR-100, we use ResNet50. Detailed settings are presented in Appendix \ref{app:backbone_frameworks}.

\paragraph{Methods} 
We consider four calibration baselines, including TS~\citep{guo2017calibration}, Vector Scaling (VS), FC~\citep{tao2025fc}, and SAG~\citep{yang2023sag}. We compare these methods with their adjusted version by LAPE. We named these adjusted baselines by TS+LAPE, VS+LAPE, FC+LAPE, and SAG+LAPE, respectively. Other baselines are also compared; see Appendix \ref{app:additional_baselines} for details.

\paragraph{Evaluation.}
Let $I_1,\dots,I_M$ partition the confidence interval $[0,1]$. For $i$-th test point, let $\hat y_i=\arg\max_y \hat p_y(x_i)$ and $\hat c_i=\max_y \hat p_y(x_i)$ be the predicted label and confidence, respectively. We evaluate calibration performance using two complementary metrics: \textit{Negative Log-Likelihood} (NLL)~\citep{gneiting2007strictly} and \textit{Expected Calibration Error} (ECE)~\citep{naeini2015obtaining}:
\begin{equation*}
    \mathrm{NLL} = -\frac{1}{N}\sum_{i=1}^{N}\log \hat p_{y_i}(x_i),\quad     \mathrm{ECE}
    =\sum_{m=1}^{M}\frac{|B_m|}{N}\left|\mathrm{acc}(B_m)-\mathrm{conf}(B_m)\right|,
\end{equation*}
where $B_m=\{i:\hat c_i\in I_m\}$, and
\begin{equation*}
\text{acc}(B_m)=\frac{1}{|B_m|}\sum_{i\in B_m}\mathbf{1}\{\hat y_i=y_i\},
\qquad
\text{conf}(B_m)=\frac{1}{|B_m|}\sum_{i\in B_m}\hat c_i.
\end{equation*}

Note that we report NLL since minimizing NLL is equivalent to minimizing the KL divergence to the true conditional distribution: $\mathbb{E}\!\left[-\log q(Y \mid X)\right]= \mathbb{E}\!\left[ H(P(\cdot \mid X))+\mathrm{KL}\!\left(P(\cdot \mid X)\,\|\,q(\cdot \mid X)\right)\right]$.
The definition and results for ACE~\citep{nixon2019measuring} are presented in Appendix~\ref{app:additional_metrics} and \ref{app:supplementary_results}, respectively.

\subsection{Results}
\label{sec:main_results_unified}

{\color{blue}}

\begin{figure}[H]
\centering
\includegraphics[width=\textwidth]{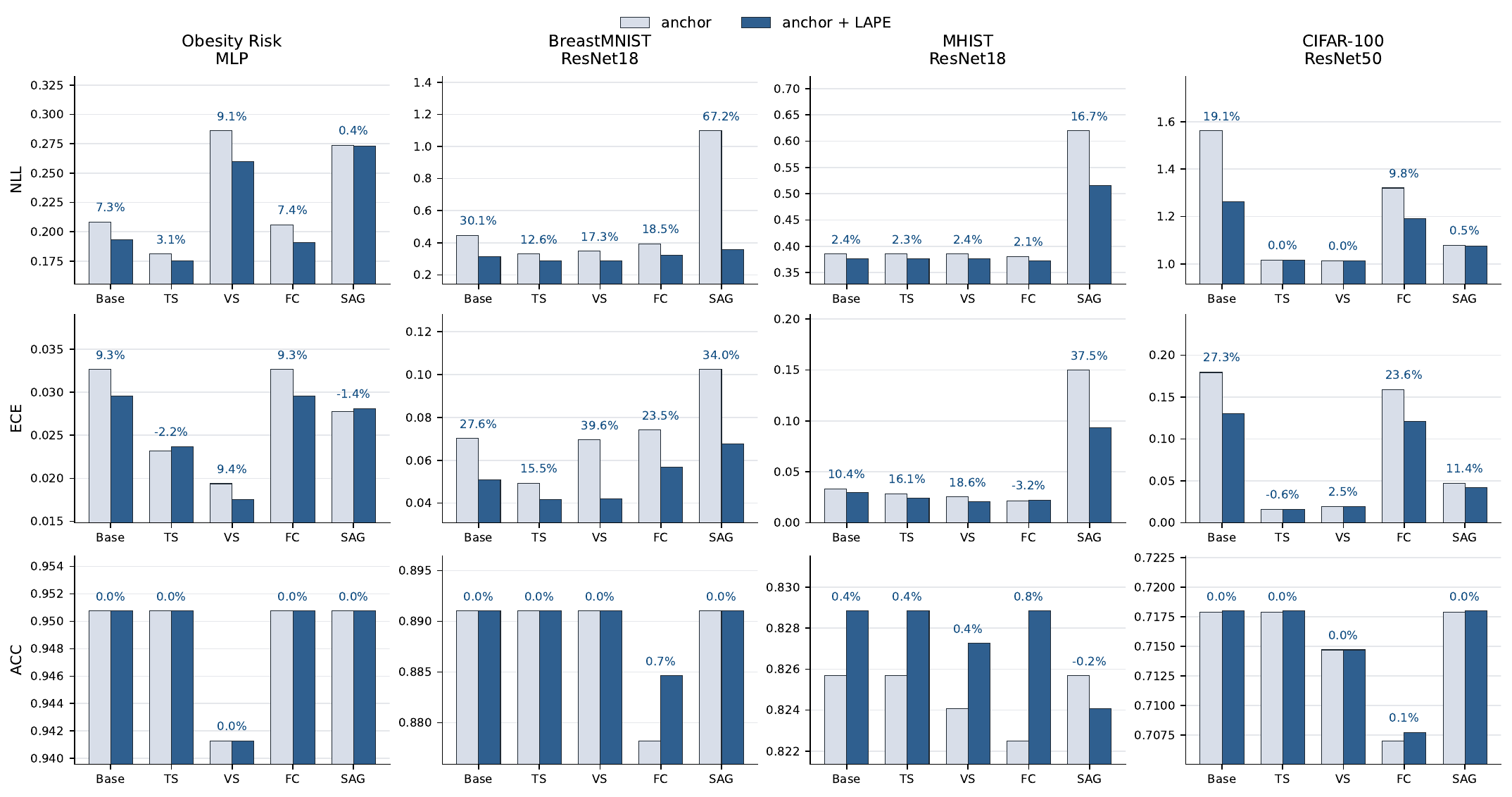}
\caption{\textbf{Paired comparison of LAPE-side refinement across calibration anchors.} Each anchor and its LAPE refinement are evaluated under the same dataset and backbone. Rows report NLL, ECE, and accuracy.
Bar annotations show relative changes, computed as \((v_{\rm anchor}-v_{\rm LAPE})/v_{\rm anchor}\times100\%\) for NLL/ECE and \((v_{\rm LAPE}-v_{\rm anchor})/v_{\rm anchor}\times100\%\) for accuracy.}
\label{fig:main_method_suite}
\end{figure}

\paragraph{LAPE improves existing post-hoc calibration methods.} Figure~\ref{fig:main_method_suite} demonstrates that LAPE can further improve a range of existing post-hoc calibration methods. Across tabular MLPs, medical-image transfer models, and natural-image backbones, adding LAPE often reduces NLL and ECE while largely preserving accuracy. This supports the role of LAPE as a residual refinement module: it keeps the global calibrator as a stable anchor and uses nearby calibration samples to correct remaining local miscalibration. The gains are not uniform across all anchors and metrics, which is expected because the benefit of local refinement depends on both the quality of the anchor and the informativeness of the learned feature geometry.

\begin{table}[H]
\centering
\caption{\textbf{Results of CFAF-LAPE and late fusion on multi-view tasks.}}
\label{tab:cfaf_main}
\small
\setlength{\tabcolsep}{4pt}
\renewcommand{\arraystretch}{1.06}
\begin{tabular}{llcccc}
\toprule
\textbf{Setting} & \textbf{Method} & \textbf{NLL} $\downarrow$ & \textbf{ECE} $\downarrow$ & \textbf{Acc.} $\uparrow$ & \textbf{Full results} \\
\midrule
Synthetic XOR views 
& Late fusion 
& 0.507 & 0.099 & 0.790
& Table~\ref{tab:cfaf_synthetic_full} \\
& CFAF-LAPE 
& \textbf{0.413} & \textbf{0.057} & \textbf{0.848}
&  \\

\midrule
Multi-view Digits 
& Late fusion 
& 0.112 & 0.029 & 0.960
& Table~\ref{tab:cfaf_digits_full} \\
& CFAF-LAPE 
& \textbf{0.092} & \textbf{0.025} & \textbf{0.964}
&  \\

\midrule
HAM10000 melanoma vs nevus 
& Late fusion 
& 0.255 & 0.023 & 0.896
& Table~\ref{tab:cfaf_ham_full} \\
& CFAF-LAPE 
& \textbf{0.224} & \textbf{0.014} & \textbf{0.907}
&  \\
\bottomrule
\end{tabular}
\end{table}

\paragraph{CFAF-LAPE improves multi-view calibration.} Table~\ref{tab:cfaf_main} examines whether local residual refinement remains useful in multi-view settings. Across all tasks, CFAF consistently improves NLL, ECE, and accuracy, and is much better than the late fusion strategy. This suggests that late fusion is insufficient if residual calibration errors depend on cross-view interactions. CFAF addresses this by forming a calibrated fusion anchor and applying local residual refinement in the rectified joint representation, without retraining an end-to-end fusion model.

\begin{figure}[H]
\centering
\includegraphics[width=\textwidth]{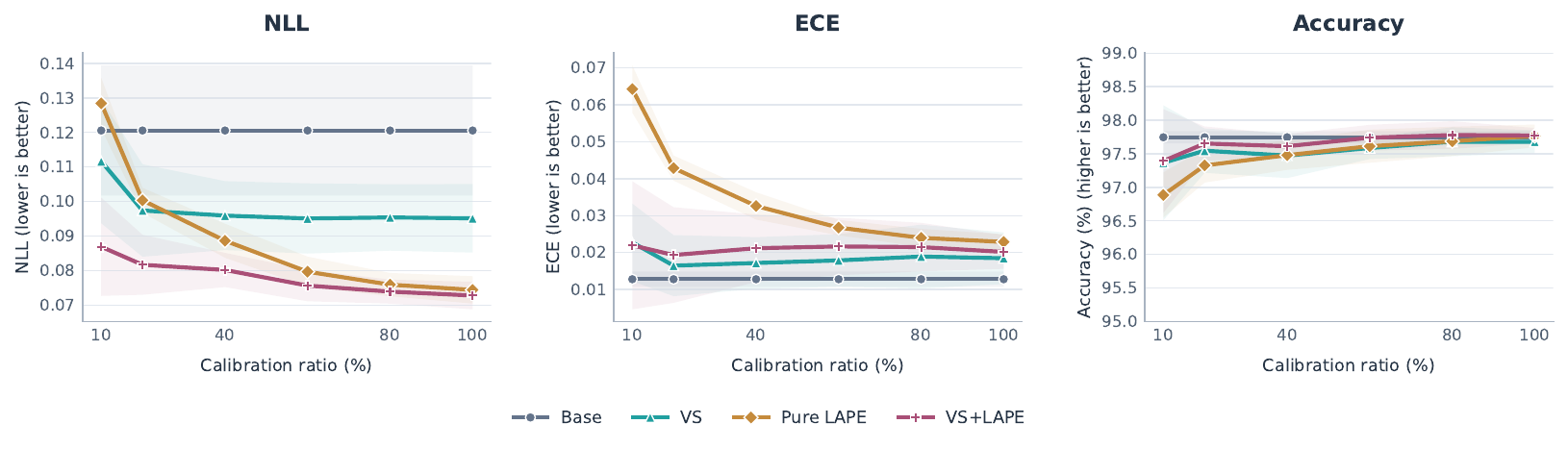}
\caption{\textbf{Sensitivity to calibration data size (calibration ratio).} NLL, ECE, and accuracy are reported, with shaded areas indicating standard deviations across five calibration split seeds.}
\label{fig:sample_ratio_sensitivity}
\end{figure}

\paragraph{Robustness to calibration-set size.}
Figure~\ref{fig:sample_ratio_sensitivity} shows that VS+LAPE lowers NLL at every ratio, even while VS still improves with more data (Table~\ref{tab:sim2_full}). As more calibration data become available, its NLL decreases further; ECE and accuracy also show favorable overall trends despite small fluctuations. These results suggest that LAPE benefits from additional calibration data but does not require a large calibration set to improve the primary NLL metric.

\paragraph{Summary of results in appendix.}

Appendix~\ref{app:supplementary_results} provides additional evidence for the advantage of LAPE.
The runtime study shows that the accelerated implementation keeps LAPE practical at inference time (Table~\ref{tab:appendix_runtime_speed}).
The heterogeneity and covariate-shift diagnostics show that LAPE is most useful when residual errors are structured in feature-space neighborhoods and when target-domain calibration samples reveal those errors; the CIFAR-10-C experiment further confirms this behavior under real-image corruption (Tables~\ref{tab:sim1_full}, \ref{tab:sim3_full}, and~\ref{tab:cifar10c_gaussian}).
The high-dimensional and structural-noise studies shows that, less reliable retrieval geometry can reduce the gains, but the refinement remains stable for most anchors (Tables~\ref{tab:sim7_full} and~\ref{tab:sim5_full_b}).
The architecture-depth diagnostic shows that LAPE benefits from usable local structure across different MLP depths, and the ablation study verifies that both adaptive neighborhood selection and anchor--local interpolation contribute to the final gains (Tables~\ref{tab:architecture_depth_full} and~\ref{tab:sim8_full}).
The CLIP and training-time calibration experiments show that LAPE remains effective beyond standard post-hoc anchors, providing further refinement after prompt tuning, DAC calibration, and calibration-aware training objectives (Tables~\ref{tab:coop_dac_lape} and~\ref{tab:train_time_lape}).
The cross-fitting experiment shows that out-of-fold anchor construction is more effective than direct reuse or simple data splitting when flexible anchors are used with limited calibration data (Table~\ref{tab:cross_fitting}).
The dynamic-fusion experiment further shows that CFAF-LAPE remains complementary to stronger learned multi-view fusion methods such as QMF and PDF (Table~\ref{tab:dynamic_fusion_lape}).
Results on more datasets and backbones further demonstrate the advantage and generality of LAPE (Tables~\ref{tab:appendix_tabular_full} and~\ref{tab:image_transfer_full}).

\section{Conclusion}
\label{sec:conclusion}

We propose LAPE, a post-hoc residual calibration framework that refines global calibrated predictions with local feature-space evidence. Instead of replacing existing calibrators, LAPE treats their outputs as stable anchors and applies a non-parametric local correction using nearby calibration samples. To reduce calibration-set overfitting while retaining data efficiency, we introduce a \(V\)-fold cross-fitting strategy for constructing out-of-sample anchor scores on the calibration set. We further propose CFAF to extend the framework to multi-view scenarios. Theoretically, we establish the asymptotic consistency of LAPE under regular assumptions. Extensive experiments demonstrate that LAPE can improve the calibration quality of diverse post-hoc baselines while remaining practical under varying calibration set sizes and distribution shifts. We expect that the proposed LAPE can serve as a basic tool towards reliable, data-efficient, and geometry-aware probability calibration.

\paragraph{Limitation.}
Although LAPE provides a flexible residual refinement framework for post-hoc calibration, its effectiveness depends on the quality of the learned feature geometry. When nearby samples in feature space do not share similar conditional label distributions, local frequency estimates may become biased and the method should rely more heavily on the global anchor. Future work may study more robust retrieval metrics.
Furthermore, extending local residual refinement to other confidence representations, such as verbalized confidence in large reasoning models \citep{wang2026pitfalls}, remains an interesting direction for future work.

\begin{ack}
This research is supported by the Zhejiang Provincial Peak Discipline Program (Statistics, Zhejiang Gongshang University), the Collaborative Innovation Center of Zhejiang Gongshang University, and the Southern University of Science and Technology.
\end{ack}

\bibliography{ref}
\bibliographystyle{plainnat}

\newpage
\appendix
\onecolumn

\section{Related Work}

\paragraph{Global score-based post-hoc calibration.}
Post-hoc calibration aims to improve the reliability of a trained classifier without retraining the base model.
Classical methods in this family include Temperature Scaling (TS), Vector Scaling (VS), Matrix Scaling (MS), Beta Calibration, Dirichlet Calibration, and ensemble or compositional variants such as Ensemble Temperature Scaling (ETS)~\citep{guo2017calibration,kull2017beyond,kull2019beyond,zhang2020mix}.
These methods are widely used because they are simple, stable, and require only a held-out calibration set.
They preserve the learned classifier and usually fit only a small number of calibration parameters, which makes them practical for modern neural networks and transfer-learning pipelines.
However, their corrections are mainly determined by logits, confidence scores, or low-dimensional transformations of the predicted probability vector.
As a result, examples with similar score patterns tend to receive similar corrections, even when they lie in different regions of the learned feature space and have different empirical reliability.
Thus, global score-based calibrators can improve average calibration but may still leave spatially structured residual errors, especially near decision boundaries, rare subpopulations, or locally mixed regions.

\paragraph{Adaptive parametric and group-wise calibration maps.}
A natural way to address the rigidity of global score transformations is to increase the expressiveness of the calibration map.
Parameterized Temperature Scaling (PTS) predicts a sample-specific temperature from the logit vector~\citep{tomani2022parameterized}, while Sample-Dependent Adaptive Temperature Scaling (AdaTS) assigns different temperatures to different inputs~\citep{joy2023sample}.
Local Temperature Scaling (LTS) further makes temperature spatially local by predicting pixel- or voxel-level temperature values for semantic segmentation~\citep{Ding_2021_ICCV}.
Semantic-Aware Grouping (SAG) follows a group-wise route: it learns a semantic grouping function from frozen deep features and logits, partitions the data into semantic subsets, and fits a separate calibrator within each group~\citep{yang2023sag}.
These methods are important steps toward instance-dependent, spatially adaptive, or group-wise calibration.
Nevertheless, their adaptivity is obtained by fitting additional calibration components, such as neural temperature predictors, local temperature fields, or semantic grouping heads.
Such extra degrees of freedom require regularization, hyperparameter selection, and sufficient held-out calibration data.
Indeed, recent studies on adaptive temperature scaling observe that highly expressive calibrators can perform well when calibration data are abundant, but may become fragile under limited calibration samples~\citep{balanya2024adaptive}.
This motivates a different route: instead of learning a richer parametric calibration map, LAPE keeps a global calibrated anchor fixed and uses held-out reference labels only as non-parametric local residual evidence.

\paragraph{Non-parametric score-level calibration.}
Non-parametric calibration has a long history in classifier probability estimation.
Histogram binning estimates empirical label frequencies within score bins, while isotonic regression learns a monotone score-to-probability map~\citep{zadrozny2001obtaining,zadrozny2002transforming}.
Bayesian Binning into Quantiles (BBQ) improves histogram binning by averaging over multiple quantile-binning models~\citep{naeini2015obtaining}.
For multiclass calibration, IROvA extends binary isotonic regression by fitting one class-wise monotone map for each one-vs-all problem, while multiclass isotonic regression, also referred to as MIR or IRM, improves data efficiency by pooling probability-label entries from all classes and learning a shared monotone map~\citep{zhang2020mix}.
Scaling-binning calibrators further combine parametric scaling with binning to balance calibration guarantees and sample efficiency~\citep{kumar2019verified}.
These methods provide flexible alternatives to global parametric maps, but their flexibility is still mainly defined over scalar scores, confidence bins, class-wise probabilities, or probability-vector entries.
They do not directly exploit the geometry of the learned feature space.
Moreover, when used as stand-alone local or partition-based probability estimators, they can be data-inefficient in finite calibration samples and may discard useful information already encoded in the neural network logits.

\paragraph{Position of LAPE.}
LAPE combines the stability of global post-hoc calibration with the flexibility of local non-parametric evidence.
Unlike global score transformations, it can distinguish examples with similar logits but different feature-space neighborhoods.
Unlike adaptive parametric or group-wise calibrators, it does not learn a new temperature predictor, spatial temperature map, or semantic partition.
Unlike score-level non-parametric methods, its local evidence is defined by neighboring reference samples in the learned feature space rather than by score bins or probability entries.
The local empirical distribution is therefore used only as a residual correction to a stable global anchor, not as a replacement for the classifier output.
Under regularity conditions, this anchor-preserving local refinement yields a consistent estimator of the conditional class probability.

Table~\ref{tab:posthoc_calibration_taxonomy} summarizes the main families of confidence calibration methods considered in this work and highlights the methodological position of LAPE.

\begin{table}[t]
\centering
\footnotesize
\caption{Main families of confidence calibration methods and the position of LAPE. 
The comparison focuses on how each method family obtains calibration adaptivity and whether it directly uses feature-space reference-label evidence.}
\label{tab:posthoc_calibration_taxonomy}
\renewcommand{\arraystretch}{1.15}
\setlength{\tabcolsep}{3pt}
\begin{tabularx}{\textwidth}{
>{\raggedright\arraybackslash}p{0.19\textwidth}
>{\raggedright\arraybackslash}p{0.25\textwidth}
>{\raggedright\arraybackslash}p{0.24\textwidth}
>{\raggedright\arraybackslash}X}
\toprule
\textbf{Method family}
& \textbf{Examples}
& \textbf{Main idea}
& \textbf{Main limitation / difference from LAPE} \\
\midrule
Global score/logit maps
& TS, VS, MS, Beta, Dirichlet, ETS~\citep{guo2017calibration,kull2017beyond,kull2019beyond,zhang2020mix}
& Fit a global transformation on logits, confidence scores, or probability vectors.
& Stable but mainly score-level; similar score patterns receive similar corrections even when feature-space neighborhoods have different empirical reliability. \\
\midrule
Adaptive parametric or group-wise maps
& PTS, AdaTS, LTS, SAG~\citep{tomani2022parameterized,joy2023sample,Ding_2021_ICCV,yang2023sag}
& Learn sample-specific temperatures, local temperature fields, or semantic group-wise calibrators.
& More adaptive, but introduces extra fitted components such as temperature predictors or grouping heads; expressive calibrators may be fragile with limited calibration data~\citep{balanya2024adaptive}. \\
\midrule
Score-level non-parametric calibration
& Histogram binning, isotonic regression, BBQ, IROvA, MIR/IRM, scaling-binning~\citep{zadrozny2001obtaining,zadrozny2002transforming,naeini2015obtaining,zhang2020mix,kumar2019verified}
& Estimate flexible calibration maps from score bins, monotone score transformations, or probability-vector entries.
& Non-parametric but still mainly score/probability-level; does not directly exploit learned feature geometry and can be data-inefficient in finite calibration samples. \\
\midrule
LAPE
& Pure LAPE; TS+LAPE; VS+LAPE; Dir+LAPE; FC+LAPE; SAG+LAPE
& Preserve a global calibrated anchor and refine it with local reference-label evidence in feature space.
& Uses non-parametric local evidence only as a residual correction, rather than learning a richer calibration map or replacing the classifier output. \\
\bottomrule
\end{tabularx}
\end{table}

\section{Technical Proofs}\label{proofs}

In this section, we provide the necessary assumptions and the proof of Theorem 5.1.

\subsection{Notations}

We use the notation introduced in Sections 2 and 3. In particular,
$\mathcal{D}_{\mathrm{cal}}=\{(x_i,y_i)\}_{i=1}^n$ is the calibration set,
$z_i=\phi(x_i)$ and $z=\phi(x)$ are the feature embeddings,
$\mathcal{N}_m(x)$ is the $m$-nearest-neighbor set in the feature space,
$\tilde p_0(x,m)$ is the smoothed local empirical probability vector,
$\pi^{\mathrm{anchor}}(x)$ is the anchor probability vector, and
$\hat p_n(x)$ denotes the final LAPE estimator constructed from a calibration
set of size $n$.

For the proof, we make the following $n$-dependent notation explicit. Let
$P_X$ denote the marginal distribution of $X$. For the retrieval embedding
$Z=\phi(X)$, let $P_Z$ and $p_Z$ denote its marginal distribution and density,
respectively, and let $\mathcal{Z}\subset\mathbb{R}^{d_z}$ denote the support
of $Z$. Let
\[
d_i(x):=\|z_i-z\|_2,
\]
and let $d_{(r)}(x)$ denote the $r$-th smallest value among
$\{d_i(x)\}_{i=1}^n$. We write
\[
m_{\sigma,n}:=\left\lfloor\frac{\sqrt n}{2}\right\rfloor,
\qquad
\sigma_n:=d_{(m_{\sigma,n})}(x).
\]
Let $\mathcal{M}_n$ and $\mathcal{A}_n$ denote the candidate grids for the
neighborhood size and interpolation weight, respectively, and write
\[
m_{\min,n}:=\min\mathcal{M}_n.
\]
The selected candidate is denoted by
\[
(m_n^*,\alpha_n^*)
\in
\arg\min_{m\in\mathcal{M}_n,\ \alpha\in\mathcal{A}_n}
\mathcal{J}_n(x;m,\alpha),
\]
where $\mathcal{J}_n(x;m,\alpha)$ is the local weighted leave-one-out negative
log-likelihood. Thus,
\[
\hat p_n(x)=\hat p_n(x;m_n^*,\alpha_n^*),
\]
where $\hat p_n(x;m,\alpha)$ denotes the corresponding LAPE candidate before
inserting the selected parameters.

\begin{definition}[Optimal Fused Bias]
\label{def:fused_bias}
For a fixed test point $x$ and candidate neighborhood size $m$, define the
conditional mean of the local empirical probability vector by
\[
\bar p_0(x,m):=\mathbb{E}\{\tilde p_0(x,m)\mid z_1,\ldots,z_n\}.
\]
The optimal fused bias of LAPE is
\[
\operatorname{Bias}_{\mathrm{LAPE}}(x;m)
:=
\min_{\alpha\in[0,1]}
\left\|
(1-\alpha)\pi^{\mathrm{anchor}}(x)
+
\alpha\bar p_0(x,m)
-
p(x)
\right\|_2 .
\]
The smoothing term in $\tilde p_0(x,m)$ is kept in the definition, but it is
asymptotically negligible when $m\to\infty$.
\end{definition}

\subsection{Assumptions}\label{assump}

\begin{assumption}[Regularity of Probability Space]
\label{ass:reg_prob}
For every $x\in\mathcal{X}$, the conditional distribution
$\mathbb{P}(Y\mid X=x)$ is well-defined, and $p(x)\in\Delta^{K-1}$. Moreover,
the feature representation used for retrieval is compatible with the target
conditional probability: there exists a measurable map
$p_{\phi}:\mathcal{Z}\to\Delta^{K-1}$ such that
\[
p(x)=p_{\phi}(\phi(x))
\]
for $P_X$-almost every $x$.
\end{assumption}

\begin{assumption}[Compact Support and Regularity]
\label{ass:compact_supp}
The retrieval support $\mathcal{Z}$ is compact, and the density $p_Z$ is
bounded away from zero and infinity on $\mathcal{Z}$:
\[
0<C_{\min}\le p_Z(z)\le C_{\max}<\infty,
\qquad z\in\mathcal{Z}.
\]
Moreover, $\mathcal{Z}$ satisfies a standard local volume regularity condition:
there exist constants $0<c_{\mathrm{vol}}\le C_{\mathrm{vol}}<\infty$ and
$r_0>0$ such that, for every $z\in\mathcal{Z}$ and every $0<r\le r_0$,
\[
c_{\mathrm{vol}}r^{d_z}
\le
\operatorname{Vol}\bigl(B(z,r)\cap\mathcal{Z}\bigr)
\le
C_{\mathrm{vol}}r^{d_z}.
\]
\end{assumption}

\begin{assumption}[Lipschitz Continuity in the Retrieval Space]
\label{ass:lipschitz_data}
The marginal density $p_Z(z)$ and the conditional probability map
$p_{\phi}(z)$ are Lipschitz continuous on $\mathcal{Z}$. That is, there exist constants $L_Z<\infty$ and $L_p<\infty$ such that, for all
$z,z'\in\mathcal{Z}$,
\[
|p_Z(z)-p_Z(z')|\le L_Z\|z-z'\|_2,
\]
\[
\|p_{\phi}(z)-p_{\phi}(z')\|_2\le L_p\|z-z'\|_2.
\]
\end{assumption}

\begin{assumption}[Lipschitz Continuity of the Anchor Predictor]
\label{ass:lipschitz_base}
The anchor probability map $\pi^{\mathrm{anchor}}$ is Lipschitz continuous in the retrieval space. Specifically, there exists a constant
$L_{\mathrm{anchor}}<\infty$ such that, for all $x,x'\in\mathcal{X}$,
\[
\|\pi^{\mathrm{anchor}}(x)-\pi^{\mathrm{anchor}}(x')\|_2
\le
L_{\mathrm{anchor}}\|\phi(x)-\phi(x')\|_2.
\]
Throughout the consistency proof, the anchor predictor is treated
as a fixed function independent of the calibration labels.
\end{assumption}

\begin{assumption}[Bounded Negative Log-Likelihood]
\label{ass:bounded_loss}
The local objective uses the negative log-likelihood $\mathcal{L}(y,q)=-\log q_y$.
Assume that there exists a constant
$\delta\in(0,1/(2K))$ such that $p_{\phi,k}(z)\geq 2\delta$ for every $z\in\mathcal{Z}$ and $k\in[K]$, and the fixed anchor predictor satisfies
\[
\pi^{\mathrm{anchor}}_k(x)\geq 2\delta,
\qquad x\in\mathcal{X},\quad k\in[K].
\]

Define the truncated log-loss $\mathcal{L}_{\delta}(y,q)
=-\log(\max\{q_y,\delta\})$.
Then
\[
0\leq\mathcal{L}_{\delta}(y,q)
\leq C_{\delta}:=-\log\delta<\infty.
\]
The truncated and ordinary log-losses coincide whenever
$q_y\geq\delta$.
\end{assumption}

\begin{assumption}[Candidate Neighborhood Grid]
\label{ass:candidate_grid}
Let
\[
m_{\min,n}:=\min\mathcal{M}_n,
\qquad
m_{\max,n}:=\max\mathcal{M}_n .
\]
The candidate neighborhood grid satisfies
\[
m_{\min,n}\to\infty,
\qquad
\log n=o(m_{\min,n}),
\qquad
m_{\max,n}\le m_{\sigma,n}.
\]
The interpolation grid satisfies $1\in\mathcal{A}_n$,
and the total grid size grows subexponentially in the smallest candidate
neighborhood size:
\[
\log(|\mathcal{M}_n||\mathcal{A}_n|)=o(m_{\min,n}).
\]
Moreover, there exists $m_n^\circ\in\mathcal{M}_n$ such that
\[
m_n^\circ\to\infty,
\qquad
\frac{m_n^\circ}{n}\to0.
\]
Thus the optimization can compare against a consistent pure-local candidate.
The practical grid used in the main algorithm satisfies these conditions.
\end{assumption}

\subsection{Preliminary Lemmas}
We introduce adaptive slices in the retrieval space.
We establish Lemma \ref{lem:geom_consist}, which ensures that the number of calibration embeddings inside each slice concentrates around its expectation. This prevents large geometric fluctuations and provides a stability prerequisite for the local empirical frequency estimator.

For the proof, let $a_n$ be an auxiliary slice-count scale satisfying
\[
a_n\to\infty,
\qquad
\frac{a_n}{m_{\min,n}}\to0,
\qquad
\log n=o(a_n).
\]
Such a sequence exists for the practical grid used in the main algorithm,
where $m_{\min,n}\asymp\sqrt n$; for example, one may take
$a_n=n^{1/4}$. This sequence is used only to define the proof slices and is
not an algorithmic tuning parameter.

\begin{lemma}[Geometric Consistency on Adaptive Slices]
\label{lem:geom_consist}
Fix a test point $x$ and write $z=\phi(x)$. Consider a partition of a
neighborhood of $z$ in the retrieval space into a finite set of disjoint slices
$\{S_j\}_{j=0}^{T_n-1}$. Let the probability mass of each slice under the
marginal distribution of $Z=\phi(X)$ satisfy
\[
c_1\frac{a_n}{n}
\le
P_Z(S_j)
\le
c_2\frac{a_n}{n},
\qquad j=0,\ldots,T_n-1,
\]
for some constants $c_1,c_2>0$. Define $\lambda_j:=nP_Z(S_j)$.
Then
\[
\lambda_j\asymp a_n,
\qquad
\lambda_j\to\infty,
\qquad
\frac{\lambda_j}{m_{\min,n}}\to0,
\]
where $m_{\min,n}:=\min\mathcal{M}_n$ is the smallest candidate neighborhood
size. Let $n_j=\sum_{i=1}^{n}\mathbb{I}(z_i\in S_j)$ be the empirical count in slice $S_j$, and let $\lambda_{\min,n}:=\min_{0\le j<T_n}\lambda_j$.
Then, for any $0<\delta<1$,
\[
\mathbb{P}\left(\max_{0\le j<T_n}|n_j-\lambda_j|>\delta\lambda_j\right)
\le
2T_n\exp\{-c\lambda_{\min,n}\delta^2\}
\]
for a universal constant $c>0$. In particular, for the adaptive slices
constructed in Lemma \ref{lem:adapt_slice}, we have
$T_n=O(n/a_n)$, and therefore the failure probability converges to zero.
\end{lemma}

\begin{proof}
For each slice $S_j$, the count $n_j$ follows a binomial distribution
$\operatorname{Bin}(n,P_Z(S_j))$ with mean $\lambda_j=nP_Z(S_j)$. By the
multiplicative Chernoff bound, for any $0<\delta<1$,
\[
\mathbb{P}(|n_j-\lambda_j|>\delta\lambda_j)
\le
2\exp\left(-\frac{\delta^2\lambda_j}{3}\right).
\]
Taking a union bound over all $T_n$ slices gives
\[
\mathbb{P}\left(\max_{0\le j<T_n}|n_j-\lambda_j|>\delta\lambda_j\right)
\le
2\sum_{j=0}^{T_n-1}\exp\left(-\frac{\delta^2\lambda_j}{3}\right)
\le
2T_n\exp\left(-\frac{\delta^2\lambda_{\min,n}}{3}\right).
\]
Since $\lambda_{\min,n}\asymp a_n$ and $a_n\to\infty$, the exponential term
vanishes. Moreover, the adaptive construction in Lemma \ref{lem:adapt_slice}
uses slices with probability mass of order $a_n/n$, so $T_n=O(n/a_n)$ and
$\log T_n=O(\log n)$. Because $\log n=o(a_n)$, the union-bound failure
probability converges to zero. Finally,
\[
\frac{\lambda_j}{m_{\min,n}}
\asymp
\frac{a_n}{m_{\min,n}}
\to0,
\]
so a single adaptive slice is asymptotically negligible relative to any
admissible candidate neighborhood size. This property will be used when
controlling partial boundary slices.
\end{proof}

We establish the next lemma. This lemma ensures that the true conditional probability remains stable inside the region bounded by two cut points. Hence, the variation within each slice is controlled and the local expectation can be approximated by its cut-point value.

\begin{lemma}[Adaptive Slicing via Probability Mass]
\label{lem:adapt_slice}
Fix a test point $x$ and write $z=\phi(x)$. For all sufficiently large $n$,
there exists a sequence of radii
\[
0=r_0<r_1<\cdots<r_{T_n}=R_{\max},
\qquad
R_{\max}:=\sup_{u\in\mathcal{Z}}\|u-z\|_2,
\]
in the retrieval space that generates the slices required by Lemma
\ref{lem:geom_consist}. Specifically, with
\[
B_z(r):=\{u\in\mathcal{Z}:\|u-z\|_2<r\},
\qquad
S_j:=B_z(r_{j+1})\setminus B_z(r_j),
\]
the slices satisfy
\[
c_1\frac{a_n}{n}
\le
P_Z(S_j)
\le
c_2\frac{a_n}{n},
\qquad j=0,\ldots,T_n-1,
\]
for some constants $c_1,c_2>0$. Moreover, the number of slices satisfies
$T_n=O(n/a_n)$.

Define the theoretical local label-frequency vector up to radius $r>0$ by
\[
\mu(r):=\mathbb{E}\{e_Y\mid Z\in B_z(r)\}
=\frac{A(r)}{V(r)},
\]
where
\[
e_Y=(\mathbb{I}(Y=1),\ldots,\mathbb{I}(Y=K)),
\]
and
\[
V(r):=\int_{B_z(r)}p_Z(u)\,du,
\qquad
A(r):=\int_{B_z(r)}p_{\phi}(u)p_Z(u)\,du.
\]
For notational convenience, set $\mu(0):=p_{\phi}(z)$.
Under Assumptions \ref{ass:compact_supp} and \ref{ass:lipschitz_data}, for sufficiently large $n$, the variation of $\mu(r)$ inside each slice satisfies
\[
\sup_{r,r'\in[r_j,r_{j+1}]}
\|\mu(r)-\mu(r')\|_2
\le
C_{\mathrm{slice}}
\left(\frac{a_n}{n}\right)^{1/d_z},
\]
where $C_{\mathrm{slice}}<\infty$ depends only on the density bounds, the Lipschitz constants, and the local volume constants of the retrieval space.
In particular, this variation vanishes as $n\to\infty$.
\end{lemma}

\begin{proof}
Let
\[
V(r)=P_Z(B_z(r))=\int_{B_z(r)}p_Z(u)\,du.
\]
Since $p_Z$ is a density bounded on the compact retrieval space, $V(r)$ is a
continuous nondecreasing function of $r$. We choose the cut points $r_j$ as
probability-mass quantiles, so that the probability increment
\[
V(r_{j+1})-V(r_j)=P_Z(S_j)
\]
is of order $a_n/n$. If the final residual slice has probability mass smaller
than the target order, we merge it with the preceding slice. After this
adjustment, all slices satisfy
\[
c_1\frac{a_n}{n}
\le
P_Z(S_j)
\le
c_2\frac{a_n}{n}
\]
for fixed constants $c_1,c_2>0$, and the number of slices is
$T_n=O(n/a_n)$.

It remains to control the variation of the conditional label-frequency vector.
Write
\[
t_n:=\frac{a_n}{n},
\qquad
h_n:=t_n^{1/d_z}.
\]
By Assumption \ref{ass:compact_supp}, there exist constants
$v_-,v_+>0$ and $r_*>0$ such that
\[
v_-r^{d_z}\le V(r)\le v_+r^{d_z},
\qquad 0<r\le r_*.
\]
Fix any $0\le r\le r'$ within the same slice. By construction,
\[
0\le\Delta V:=V(r')-V(r)\le c_2t_n.
\]
If $\Delta V=0$, then $\mu(r')=\mu(r)$. Otherwise, define
\[
\nu(r,r')
:=
\frac{A(r')-A(r)}{\Delta V}.
\]
Writing $\Delta A:=A(r')-A(r)$, we obtain
\[
\mu(r')-\mu(r)
=
\frac{\Delta V}{V(r')}
\bigl(\nu(r,r')-\mu(r)\bigr).
\]
This identity also holds for $r=0$ under the convention
$\mu(0)=p_\phi(z)$.
First suppose $r'\ge r_*$. Since $V(r')\ge V(r_*)>0$
and both $\nu(r,r')$ and $\mu(r)$ are probability vectors,
\[
\|\mu(r')-\mu(r)\|_2
\le
\frac{2\Delta V}{V(r')}
\le
\frac{2c_2t_n}{V(r_*)}
=O(h_n).
\]

Now suppose $r'<r_*$. By Assumption
\ref{ass:lipschitz_data}, both $\nu(r,r')$ and
$\mu(r)$ differ from $p_\phi(z)$ by at most
$L_pr'$. Therefore,
\[
\|\nu(r,r')-\mu(r)\|_2
\le 2L_pr'.
\]

If $r'\le h_n$, then
\[
\|\mu(r')-\mu(r)\|_2
\le 2L_pr'
\le 2L_ph_n.
\]

If $h_n<r'<r_*$, the lower volume bound gives
\[
V(r')\ge v_-(r')^{d_z}.
\]
Consequently,
\[
\|\mu(r')-\mu(r)\|_2
\le
\frac{\Delta V}{V(r')}
\,2L_pr'
\le
\frac{2L_pc_2}{v_-}
\frac{t_n}{(r')^{d_z-1}}
\le
\frac{2L_pc_2}{v_-}
\frac{t_n}{h_n^{d_z-1}}
=
\frac{2L_pc_2}{v_-}h_n.
\]
Combining the cases yields
\[
\sup_{r,r'\in[r_j,r_{j+1}]}
\|\mu(r)-\mu(r')\|_2
\le
C_{\mathrm{slice}}
\left(\frac{a_n}{n}\right)^{1/d_z}
\]
uniformly over all slices.
\end{proof}

The following lemma establishes that the local empirical frequency is anchored to the corresponding theoretical expectation at every admissible cut point.

\begin{lemma}[Stability of Empirical Frequency at Cut Points]
\label{lem:freq_stab}
For a cut point $r_\ell$, let
\[
\mathcal{I}_\ell
:=
\{i\in[n]: z_i\in B_z(r_\ell)\}
=
\{i\in[n]: z_i\in S_0\cup\cdots\cup S_{\ell-1}\},
\]
and let $n_\ell^{\mathrm{cut}}:=|\mathcal{I}_\ell|$. Define the empirical
label-frequency vector aggregated up to the cut point $r_\ell$ by
\[
\hat p_{\mathrm{slice}}(r_\ell)
:=
\frac{1}{n_\ell^{\mathrm{cut}}}
\sum_{i\in\mathcal{I}_\ell} e_{y_i},
\qquad
e_{y_i}:=(\mathbb{I}(y_i=1),\ldots,\mathbb{I}(y_i=K)).
\]
Let $\mu(r_\ell)$ be the theoretical local label-frequency vector defined in
Lemma \ref{lem:adapt_slice}. Under the geometric stability event from Lemma
\ref{lem:geom_consist}, for any $\delta>0$,
\[
\mathbb{P}\left(
\max_{\ell:\, n_\ell^{\mathrm{cut}}\ge m_{\min,n}}
\|\hat p_{\mathrm{slice}}(r_\ell)-\mu(r_\ell)\|_2>\delta
\right)\to0 .
\]
Equivalently, the aggregated empirical frequency at every admissible cut point
converges uniformly to its corresponding theoretical local expectation.
\end{lemma}

\begin{proof}
Fix a cut point $r_\ell$ such that
$n_\ell^{\mathrm{cut}}\ge m_{\min,n}$. Let
$\mathcal{F}_n=\sigma(z_1,\ldots,z_n)$ be the sigma-algebra generated by the
feature embeddings. Conditional on $\mathcal{F}_n$, the labels are independent,
and
\[
\mathbb{E}[e_{y_i}\mid \mathcal{F}_n]=p_\phi(z_i).
\]
Define the conditional average probability vector inside the cut region by
\[
\bar p_\ell
:=
\frac{1}{n_\ell^{\mathrm{cut}}}
\sum_{i\in\mathcal{I}_\ell} p_\phi(z_i).
\]
Then
\[
\hat p_{\mathrm{slice}}(r_\ell)-\bar p_\ell
=
\frac{1}{n_\ell^{\mathrm{cut}}}
\sum_{i\in\mathcal{I}_\ell}
\left(e_{y_i}-p_\phi(z_i)\right).
\]
Each coordinate of $e_{y_i}-p_\phi(z_i)$ is centered and bounded in $[-1,1]$.
By applying Hoeffding's inequality coordinatewise and taking a union bound
over the $K$ coordinates, we obtain, for any $\delta>0$,
\[
\mathbb{P}\left(
\|\hat p_{\mathrm{slice}}(r_\ell)-\bar p_\ell\|_2>\delta
\mid \mathcal{F}_n
\right)
\le
2K\exp\left(
-\frac{2n_\ell^{\mathrm{cut}}\delta^2}{K}
\right).
\]
Since $n_\ell^{\mathrm{cut}}\ge m_{\min,n}$ and
$m_{\min,n}\to\infty$, the stochastic deviation vanishes. Taking a union bound
over the at most $T_n$ cut points, with $T_n$ growing at most polynomially,
gives
\[
\mathbb{P}\left(
\max_{\ell:\, n_\ell^{\mathrm{cut}}\ge m_{\min,n}}
\|\hat p_{\mathrm{slice}}(r_\ell)-\bar p_\ell\|_2>\delta
\mid \mathcal{F}_n
\right)
\le
2KT_n\exp\left(
-\frac{2m_{\min,n}\delta^2}{K}
\right)
\to0 .
\]

It remains to compare $\bar p_\ell$ with the theoretical local expectation
$\mu(r_\ell)$. The cut points $r_\ell$ are deterministic population
quantiles for fixed $x$ and $n$. Conditional on the membership indicators
$\{\mathbb{I}(z_i\in B_z(r_\ell))\}_{i=1}^n$, the embeddings inside
$B_z(r_\ell)$ are independent draws from
$P_Z(\cdot\mid Z\in B_z(r_\ell))$.
Therefore, we have
\[
\mathbb{E}[p_\phi(Z)\mid Z\in B_z(r_\ell)]
=
\mu(r_\ell).
\]
Since every coordinate of $p_\phi(Z)$ lies in $[0,1]$,
Hoeffding's inequality gives, for any $\delta>0$,
\[
\begin{aligned}
\mathbb{P}\left(
\|\bar p_\ell-\mu(r_\ell)\|_2>\delta,\,
n_\ell^{\mathrm{cut}}\ge m_{\min,n}
\right)\le
2K\exp\left(
-\frac{2m_{\min,n}\delta^2}{K}
\right).
\end{aligned}
\]
Taking a union bound over the $T_n$ cut points yields
\[
\begin{aligned}
\mathbb{P}\left(
\max_{\ell:\, n_\ell^{\mathrm{cut}}\ge m_{\min,n}}
\|\bar p_\ell-\mu(r_\ell)\|_2>\delta
\right)\le
2KT_n\exp\left(
-\frac{2m_{\min,n}\delta^2}{K}
\right)
\longrightarrow0.
\end{aligned}
\]
Combining the two concentration bounds with the triangle inequality yields
\[
\mathbb{P}\left(
\max_{\ell:\, n_\ell^{\mathrm{cut}}\ge m_{\min,n}}
\|\hat p_{\mathrm{slice}}(r_\ell)-\mu(r_\ell)\|_2>\delta
\right)\to0 .
\]
This proves the lemma.
\end{proof}

We establish Lemma \ref{lem:bandwidth}. This lemma describes the asymptotic behavior of the weighting bandwidth used in the local empirical loss. It shows that the bandwidth $\sigma_n$ shrinks to zero, while the number of calibration embeddings inside this radius diverges.

\begin{lemma}[Bandwidth and Sample Size Asymptotics]
\label{lem:bandwidth}
Let
\[
m_{\sigma,n}:=\left\lfloor\frac{\sqrt n}{2}\right\rfloor,
\qquad
\sigma_n=d_{(m_{\sigma,n})}(x),
\]
where $d_{(m_{\sigma,n})}(x)$ is the $m_{\sigma,n}$-th nearest-neighbor distance from $z=\phi(x)$ among the calibration embeddings $\{z_i\}_{i=1}^n$.
Under Assumption \ref{ass:compact_supp} and the local volume regularity used in Lemma \ref{lem:adapt_slice},
\[
\sigma_n=\Theta_p\left(n^{-1/(2d_z)}\right),
\]
and hence $\sigma_n\to0$ in probability. Moreover, the effective local sample size
\[
N_{\mathrm{eff}}(\sigma_n):=
\sum_{i=1}^n \mathbb{I}\{\|z_i-z\|_2\le \sigma_n\}
\]
satisfies $N_{\mathrm{eff}}(\sigma_n)\ge m_{\sigma,n}=\Theta(\sqrt n)$, and therefore $N_{\mathrm{eff}}(\sigma_n)\to\infty$.
\end{lemma}

\begin{proof}
We first prove the Convergence of $\sigma_n$.
For a fixed test embedding $z$, define
\[
V(r):=P_Z(B_z(r)).
\]
By Assumption \ref{ass:compact_supp}, the density bounds and local
volume regularity imply that there exist constants
$c_-,c_+>0$ and $r_0>0$ such that
\[
c_-r^{d_z}\le V(r)\le c_+r^{d_z},
\qquad 0<r\le r_0.
\]
Let $t_n:=\left(m_{\sigma,n}/n\right)^{1/d_z}$.
Choose fixed constants $a,b>0$ satisfying
\[
c_+a^{d_z}\le\frac12,
\qquad
c_-b^{d_z}\ge2.
\]
Since $m_{\sigma,n}/n\to0$, both $at_n$ and $bt_n$
lie within the local regularity radius for sufficiently
large $n$.
Define the number of calibration embeddings within
radius $r$ by
\[
N_n(r):=
\sum_{i=1}^n
\mathbb{I}\{\|z_i-z\|_2\le r\}.
\]
For each fixed $r$, we have $N_n(r)\sim\operatorname{Bin}(n,V(r))$.
The volume bounds yield
\[
\mathbb{E}N_n(at_n)
\le \frac{m_{\sigma,n}}{2},
\qquad
\mathbb{E}N_n(bt_n)
\ge 2m_{\sigma,n}.
\]
By the multiplicative Chernoff bounds, there exists
a constant $c>0$ such that
\[
\mathbb{P}\bigl(N_n(at_n)\ge m_{\sigma,n}\bigr)
\le \exp(-cm_{\sigma,n}),
\quad
\mathbb{P}\bigl(N_n(bt_n)<m_{\sigma,n}\bigr)
\le \exp(-cm_{\sigma,n}).
\]
Consequently,
\[
\mathbb{P}\left(
at_n\le\sigma_n\le bt_n
\right)
\ge
1-2\exp(-cm_{\sigma,n})
\longrightarrow1.
\]
Since $m_{\sigma,n}=\Theta(n^{1/2})$, it follows that $t_n=\Theta\left(n^{-1/(2d_z)}\right)$.
Therefore, we have $\sigma_n=\Theta_p\left(n^{-1/(2d_z)}\right)$,
which implies $\sigma_n\to0$ in probability.
We then prove the divergence of the effective local sample size.
By definition of the order statistic $d_{(m_{\sigma,n})}(x)$, at least $m_{\sigma,n}$ calibration embeddings lie inside the closed ball of radius $\sigma_n$ around $z$. Hence
\[
N_{\mathrm{eff}}(\sigma_n)
=
\sum_{i=1}^n \mathbb{I}\{\|z_i-z\|_2\le \sigma_n\}
\ge
m_{\sigma,n}
=
\Theta(\sqrt n).
\]
Thus $N_{\mathrm{eff}}(\sigma_n)\to\infty$ as $n\to\infty$.
\end{proof}

We now establish Lemma \ref{lem:cond_unif_conv}. This lemma shows that, on the
candidate grids used by LAPE, the local weighted leave-one-out negative
log-likelihood concentrates around its conditional mean. The key point is that
all candidate neighborhoods contain at least $m_{\min,n}$ points, while the
Gaussian weights are uniformly of order $1/m$ on each candidate neighborhood.

\begin{lemma}[Conditional Uniform Convergence of Empirical Risk]
\label{lem:cond_unif_conv}
For $m\in\mathcal{M}_n$ and $\alpha\in\mathcal{A}_n$, write the local weighted
leave-one-out objective as
\[
\mathcal{J}_n(x;m,\alpha)
=
\sum_{i\in\mathcal{N}_m(x)}
w_i^{(m)}(x)\,
\mathcal{L}\bigl(y_i,q_{i,n}^{(-i)}(m,\alpha)\bigr),
\]
where $q_{i,n}^{(-i)}(m,\alpha)$ denotes the leave-one-out probability vector used to evaluate the calibration point $i$. Here $q_{i,n}^{(-i)}(m,\alpha)$ is constructed from the same query-local pool $\mathcal{N}_m(x)\setminus\{i\}$, so for fixed $(m,\alpha)$ the objective depends only on the labels inside $\mathcal{N}_m(x)$, with the anchor values treated as fixed.
Assume that the candidate grid satisfies $m\le m_{\sigma,n}$ for all $m\in\mathcal{M}_n$, as in the main algorithm.
Let
\[
\mathcal{G}_n
:=
\sigma\bigl(
z_1,\ldots,z_n,\,
\pi^{\mathrm{anchor}}(x),\,
\pi^{\mathrm{anchor}}(x_1),\ldots,\pi^{\mathrm{anchor}}(x_n)
\bigr)
\]
be the sigma-algebra generated by the embeddings and the fixed anchor values.
Define the corresponding conditional risk by
\[
\mathcal{R}_n(x;m,\alpha)
:=
\mathbb{E}[\mathcal{J}_n(x;m,\alpha)\mid \mathcal{G}_n].
\]
Under Assumptions \ref{ass:compact_supp}--\ref{ass:bounded_loss} and the
leave-one-out construction, for every fixed $\varepsilon>0$,
\[
\mathbb{P}\left(
\sup_{m\in\mathcal{M}_n,\,\alpha\in\mathcal{A}_n}
\left|\mathcal{J}_n(x;m,\alpha)-\mathcal{R}_n(x;m,\alpha)\right|
>
\varepsilon
\right)
\to0,
\]
provided that $|\mathcal{M}_n||\mathcal{A}_n|$ grows subexponentially in $m_{\min,n}$.
In particular, this holds for the finite or polynomial-size grids used in the algorithm.
\end{lemma}

\begin{proof}
Fix $m\in\mathcal{M}_n$ and $\alpha\in\mathcal{A}_n$.
Condition on $\mathcal{G}_n$. Under Assumptions B.2
and B.5, the calibration labels are conditionally
independent, with
\[
\mathbb{P}(Y_i=k\mid\mathcal{G}_n)
=p_{\phi,k}(z_i).
\]
For any $i\in\mathcal{N}_m(x)$ and $k\in[K]$,
the leave-one-out local frequency is
\[
\widetilde p_{-i}(k;x_i,m)
=
\frac{
\sum_{j\in\mathcal{N}_m(x)\setminus\{i\}}
\mathbf{1}\{Y_j=k\}+1/K
}{m}.
\]
By Assumption B.6,
\[
\mathbb{E}\left[
\widetilde p_{-i}(k;x_i,m)
\mid\mathcal{G}_n
\right]
\geq
\frac{2\delta(m-1)}{m}.
\]
Since $m\geq m_{\min,n}\to\infty$,
Hoeffding's inequality yields
\[
\mathbb{P}\left(
\widetilde p_{-i}(k;x_i,m)<\delta
\mid\mathcal{G}_n
\right)
\leq \exp(-c_\delta m)
\]
for some constant $c_\delta>0$ and all
sufficiently large $n$.
Let
\[
\mathcal{E}_n=
\left\{
\min_{\substack{
m\in\mathcal{M}_n,
i\in\mathcal{N}_m(x),
k\in[K]}}
\widetilde p_{-i}(k;x_i,m)
\geq\delta
\right\}.
\]
A union bound gives
\[
\mathbb{P}(\mathcal{E}_n^c\mid\mathcal{G}_n)
\leq
Kn|\mathcal{M}_n|
\exp(-c_\delta m_{\min,n})
=o(1),
\]
where the convergence follows from
Assumption B.7.
Since every anchor coordinate is at least
$2\delta$, all fused leave-one-out candidates
have coordinates at least $\delta$ on
$\mathcal{E}_n$.
Define
\[
J_n^\delta(x;m,\alpha)
=
\sum_{i\in\mathcal{N}_m(x)}
w_i^{(m)}(x)
\mathcal{L}_{\delta}
\bigl(Y_i,q_{i,n}^{(-i)}(m,\alpha)\bigr),
\]
\[
R_n^\delta(x;m,\alpha)
=
\mathbb{E}\left[
J_n^\delta(x;m,\alpha)
\mid\mathcal{G}_n
\right].
\]
Since $m\leq m_{\sigma,n}$, we have $w_i^{(m)}(x)\leq e/m$.
Moreover, $\mathcal{L}_{\delta}$ is bounded
by $C_\delta$ and is $1/\delta$-Lipschitz
in the corresponding probability coordinate.

Changing a single label changes the direct
loss contribution by at most $eC_\delta/m$.
The total effect on the remaining
leave-one-out losses is bounded by
$2/(\delta m)$. Therefore, the
bounded-difference constant satisfies $b_m\leq (eC_\delta+2/\delta)/m$.
By McDiarmid's inequality,
\[
\mathbb{P}\left(
|J_n^\delta(x;m,\alpha)
-R_n^\delta(x;m,\alpha)|
>\varepsilon
\mid\mathcal{G}_n
\right)
\leq
2\exp\left(
-\frac{2\varepsilon^2m}
{(eC_\delta+2/\delta)^2}
\right).
\]
Taking a union bound over
$\mathcal{M}_n\times\mathcal{A}_n$ yields
\[
\mathbb{P}\left(
\sup_{\substack{
m\in\mathcal{M}_n
\alpha\in\mathcal{A}_n}}
|J_n^\delta-R_n^\delta|
>\varepsilon
\mid\mathcal{G}_n
\right)\leq
2|\mathcal{M}_n||\mathcal{A}_n|
\exp\left(
-\frac{2\varepsilon^2m_{\min,n}}
{(eC_\delta+2/\delta)^2}
\right)
\longrightarrow0.
\]
By Laplace smoothing, every leave-one-out local probability is at least $1/(Km)$.
Consequently, the ordinary log-loss is bounded by
\[
B_n=\log(Kn)=O(\log n)
\]
uniformly over the candidate grid.
On $\mathcal{E}_n$, we have $J_n(x;m,\alpha) =J_n^\delta(x;m,\alpha)$ for every candidate.
Furthermore, we have
\[
\sup_{\substack{
m\in\mathcal{M}_n
\alpha\in\mathcal{A}_n}}
|R_n-R_n^\delta|
\leq
(B_n+C_\delta)
\mathbb{P}(\mathcal{E}_n^c
\mid\mathcal{G}_n)
=o(1).
\]
Combining these bounds gives
\[
\sup_{\substack{
m\in\mathcal{M}_n
\alpha\in\mathcal{A}_n}}
|J_n(x;m,\alpha)-R_n(x;m,\alpha)|
\xrightarrow{P}0.
\]
\end{proof}

We establish Lemma \ref{lem:risk_inf}. This lemma gives the risk gap used in later contradiction arguments: if the conditional mean of a candidate probability vector stays away from the true conditional probability, then its expected negative log-likelihood is strictly above the entropy floor.

\begin{lemma}[Risk Inflation]
\label{lem:risk_inf}

Fix a test point $x$ and let $Q(x)$ be any possibly random
probability-vector estimator with strictly positive coordinates
and integrable negative log-loss.
Let $\mathcal F$ be a sigma-algebra such that, conditional on
$\mathcal F$, the label $Y$ has distribution $p(x)$ and is
independent of $Q(x)$. Define $\bar q(x):=\mathbb{E}[Q(x)\mid\mathcal F]$.
If $\|\bar q(x)-p(x)\|_2\ge c$ for some constant $c>0$,
then the conditional expected negative log-likelihood satisfies
\[
\mathbb{E}[-\log Q_Y(x)\mid\mathcal{F}]
\ge
H(p(x))+\frac{c^2}{2},
\]
where $Y\mid X=x\sim p(x)$ and $H(p(x)):=-\sum_{k=1}^K p(k;x)\log p(k;x)$.
\end{lemma}

\begin{proof}
Let $p_k:=p(k;x)$ and write $Q_k(x)$ for the $k$-th coordinate of $Q(x)$.
Conditioning on $\mathcal{F}$, the expected negative log-likelihood is
\[
\mathbb{E}[-\log Q_Y(x)\mid\mathcal{F}]
=
\sum_{k=1}^K p_k\,\mathbb{E}[-\log Q_k(x)\mid\mathcal{F}].
\]
Since $-\log t$ is convex on $(0,\infty)$, Jensen's inequality gives
\[
\mathbb{E}[-\log Q_k(x)\mid\mathcal{F}]
\ge
-\log \mathbb{E}[Q_k(x)\mid\mathcal{F}]
=
-\log \bar q_k(x).
\]
Therefore,
\[
\mathbb{E}[-\log Q_Y(x)\mid\mathcal{F}]
\ge
-\sum_{k=1}^K p_k\log \bar q_k(x)
=
H(p(x))+D_{\mathrm{KL}}\{p(x)\|\bar q(x)\}.
\]
By Pinsker's inequality,
\[
D_{\mathrm{KL}}\{p(x)\|\bar q(x)\}
\ge
2\|p(x)-\bar q(x)\|_{\mathrm{TV}}^2
=
\frac{1}{2}\|p(x)-\bar q(x)\|_1^2.
\]
Since $\|v\|_1\ge\|v\|_2$ for any finite-dimensional vector $v$, the bias
condition implies
\[
D_{\mathrm{KL}}\{p(x)\|\bar q(x)\}
\ge
\frac{1}{2}\|p(x)-\bar q(x)\|_2^2
\ge
\frac{c^2}{2}.
\]
Combining the preceding displays yields
\[
\mathbb{E}[-\log Q_Y(x)\mid\mathcal{F}]
\ge
H(p(x))+\frac{c^2}{2}.
\]
This completes the proof.
\end{proof}

We establish Lemma \ref{lem:local_weight}. This lemma shows that the local weights used by LAPE concentrate inside any fixed feature-space neighborhood of the query point. Since the largest candidate neighborhood radius is bounded by the bandwidth radius $\sigma_n$, Lemma \ref{lem:bandwidth} implies that all weighted calibration samples become local asymptotically.

\begin{lemma}[Localization of Local Gaussian Weights]
\label{lem:local_weight}
Fix a test point $x$ and write $z=\phi(x)$. For any candidate neighborhood
size $m\in\mathcal{M}_n$ and any fixed radius $r>0$, define the external
weight contribution by
\[
W_{\mathrm{out}}^{(m)}(r)
:=
\sum_{i\in\mathcal{N}_m(x):\,\|z_i-z\|_2>r} w_i^{(m)}(x).
\]
Assume that $m\le m_{\sigma,n}$ for all $m\in\mathcal{M}_n$, as in the main
algorithm. Then, under the conditions of Lemma \ref{lem:bandwidth},
\[
\sup_{m\in\mathcal{M}_n} W_{\mathrm{out}}^{(m)}(r)
\xrightarrow{P} 0.
\]
\end{lemma}

\begin{proof}
By definition,
\[
m_{\sigma,n}=\left\lfloor\frac{\sqrt n}{2}\right\rfloor,
\qquad
\sigma_n=d_{(m_{\sigma,n})}(x),
\]
where $d_{(m_{\sigma,n})}(x)$ is the $m_{\sigma,n}$-th nearest-neighbor
distance from $z$ among the calibration embeddings. If $m\le m_{\sigma,n}$,
then every point in $\mathcal{N}_m(x)$ has distance at most $\sigma_n$ from
$z$. Hence, for any fixed $r>0$, on the event $\{\sigma_n\le r\}$,
\[
\mathcal{N}_m(x)\subseteq B_z(r)
\qquad \text{for all } m\in\mathcal{M}_n.
\]
Consequently,
\[
W_{\mathrm{out}}^{(m)}(r)=0
\qquad \text{for all } m\in\mathcal{M}_n
\]
on this event. Therefore,
\[
\sup_{m\in\mathcal{M}_n} W_{\mathrm{out}}^{(m)}(r)
\le
\mathbb{I}\{\sigma_n>r\}.
\]
Lemma \ref{lem:bandwidth} gives $\sigma_n\xrightarrow{P}0$, and thus
$\mathbb{P}(\sigma_n>r)\to0$ for every fixed $r>0$. This proves $\sup_{m\in\mathcal{M}_n} W_{\mathrm{out}}^{(m)}(r)
\xrightarrow{P}0$.
\end{proof}

\subsection{Proof of Theorem 5.1}

We prove the consistency statement for a fixed test point $x$. The proof has
three parts. First, we show that the local empirical frequency is uniformly
stable over the candidate neighborhood grid. Second, we prove that the
empirical risk minimization step cannot select a candidate whose conditional
mean remains separated from the true conditional probability. Finally, we
combine these two facts to obtain convergence of the final LAPE estimator.

\textbf{Step 1: Variance control.}

\begin{proposition}[The good set $\Omega_{\mathrm{good}}(x,n)$]
\label{prop:good_set}
For $m\in\mathcal{M}_n$, let $\bar p_0(x,m)$ be the conditional mean defined in
Definition~\ref{def:fused_bias}.
For any fixed $\varepsilon>0$, define
\[
\Omega_{\mathrm{good}}(x,n;\varepsilon)
:=
\left\{
\sup_{m\in\mathcal{M}_n}
\|\tilde p_0(x,m)-\bar p_0(x,m)\|_2\le \varepsilon
\right\}.
\]
Under the assumptions above,
\[
\mathbb{P}(\Omega_{\mathrm{good}}(x,n;\varepsilon))\to1.
\]
Equivalently,
\[
\sup_{m\in\mathcal{M}_n}
\|\tilde p_0(x,m)-\bar p_0(x,m)\|_2=o_p(1).
\]
\end{proposition}

\begin{proof}[Proof of Proposition \ref{prop:good_set}]
Fix a test point $x$ and condition on $\mathcal{F}_n:=\sigma(z_1,\ldots,z_n)$. Under i.i.d. assumption and Assumption \ref{ass:reg_prob}, the labels $Y_1,\ldots,Y_n$ are conditionally independent given $\mathcal{F}_n$, with
\[
\mathbb{E}[e_{Y_i}\mid\mathcal{F}_n]=p_{\phi}(z_i),
\qquad i\in[n].
\]
Since $P_Z$ has a density, the distances from the fixed query embedding $z=\phi(x)$ are almost surely distinct. Thus $|\mathcal{N}_m(x)|=m$ for every $m\in\mathcal{M}_n$, and the neighbor sets are determined by $\mathcal{F}_n$.

By the definition of the Laplace-smoothed local frequency and its
conditional mean,
\[
\tilde p_0(x,m)
=\frac{\sum_{i\in\mathcal{N}_m(x)}e_{Y_i}+u_K}{m+1},
\qquad
\bar p_0(x,m)
=\frac{\sum_{i\in\mathcal{N}_m(x)}p_{\phi}(z_i)+u_K}{m+1},
\]
where $u_K=(1/K,\ldots,1/K)$. Consequently,
\[
\tilde p_0(x,m)-\bar p_0(x,m)
=\frac{1}{m+1}\sum_{i\in\mathcal{N}_m(x)}
\bigl(e_{Y_i}-p_{\phi}(z_i)\bigr).
\]
For each class $k\in[K]$, the summands
$\mathbb{I}(Y_i=k)-p_{\phi,k}(z_i)$ are conditionally independent,
centered, and have range length one. Hoeffding's inequality therefore gives, for every $\varepsilon>0$,
\begin{equation*}
    \mathbb{P}\left(
\left|\tilde p_0(k;x,m)-\bar p_0(k;x,m)\right|
>\frac{\varepsilon}{\sqrt K}
\,\middle|\,\mathcal{F}_n\right)\le
2\exp\left(-\frac{2(m+1)^2\varepsilon^2}{mK}\right)
\le 2\exp\left(-\frac{2m\varepsilon^2}{K}\right).
\end{equation*}
Applying a union bound over the $K$ coordinates and all
$m\in\mathcal{M}_n$ yields
\[
\mathbb{P}\left(
\sup_{m\in\mathcal{M}_n}
\|\tilde p_0(x,m)-\bar p_0(x,m)\|_2>\varepsilon
\,\middle|\,\mathcal{F}_n\right)\le
2K|\mathcal{M}_n|\exp\left(
-\frac{2m_{\min,n}\varepsilon^2}{K}\right).
\]
By Assumption \ref{ass:candidate_grid},
$\log|\mathcal{M}_n|=o(m_{\min,n})$ and
$m_{\min,n}\to\infty$. Hence the right-hand side converges to zero.
Taking expectations over $\mathcal{F}_n$ gives
\[
\mathbb{P}\left(
\sup_{m\in\mathcal{M}_n}
\|\tilde p_0(x,m)-\bar p_0(x,m)\|_2>\varepsilon
\right)\longrightarrow0.
\]
This proves the proposition.
\end{proof}

\textbf{Step 2: Bias convergence.}
\begin{proposition}[Bias convergence of the selected candidate] 
\label{thm:bias_conv} 
Let $(m_n^*,\alpha_n^*)$ be any minimizer of 
$\mathcal{J}_n(x;m,\alpha)$ over 
$\mathcal{M}_n\times\mathcal{A}_n$. Define the conditional mean of the fused 
candidate by 
\[ 
\bar q_n(x;m,\alpha) 
:= 
(1-\alpha)\pi^{\mathrm{anchor}}(x)+\alpha\bar p_0(x,m). 
\] 
Then $\|\bar q_n(x;m_n^*,\alpha_n^*)-p(x)\|_2=o_P(1)$.
Consequently, $\operatorname{Bias}_{\mathrm{LAPE}}(x;m_n^*)=o_P(1)$.
\end{proposition} 
 
\begin{proof}
We argue by contradiction. Suppose that there exist constants $c>0$ and
$\rho>0$ and a subsequence, still indexed by $n$, such that
\[
\mathbb{P}\left(
\|\bar q_n(x;m_n^*,\alpha_n^*)-p(x)\|_2\ge c
\right)\ge \rho.
\]
Consider this event together with the good event from Proposition
\ref{prop:good_set} and the uniform risk-concentration event from Lemma
\ref{lem:cond_unif_conv}, both of which have probability tending to one.

We first lower-bound the risk of such a persistently biased selected
candidate. Because $p_\phi$ and $\pi^{\mathrm{anchor}}$ are Lipschitz in the
retrieval space, and because Lemma \ref{lem:local_weight} localizes the
weights inside arbitrarily small balls around $z$, define the conditional mean
of the leave-one-out candidate from Lemma \ref{lem:cond_unif_conv} by
\[
\bar q_{i,n}^{(-i)}(m,\alpha)
:=
\mathbb{E}\!\left[
q_{i,n}^{(-i)}(m,\alpha)\mid\mathcal{G}_n
\right].
\]

For any $i\in\mathcal N_{m_n^*}(x)$, Lipschitz continuity and
the leave-one-out construction give
\[
\begin{aligned}
&\left\|
\bar q_{i,n}^{(-i)}(m_n^*,\alpha_n^*)-p_\phi(z_i)
\right\|_2\\
&\quad\ge
\left\|
\bar q_n(x;m_n^*,\alpha_n^*)-p_\phi(z)
\right\|_2
-(L_{\rm anchor}+L_p)\|z_i-z\|_2
-\frac{C}{m_n^*},
\end{aligned}
\]
for some constant $C>0$. By Lemma \ref{lem:bandwidth},
$\|z_i-z\|_2\le\sigma_n=o_p(1)$ uniformly over the
neighborhood, while $m_n^*\ge m_{\min,n}\to\infty$.
Thus, on the event of a bias of at least $c$ at $x$, the
left-hand side is at least $c/2$ with probability tending
to one. Lemma \ref{lem:risk_inf} then yields a positive
excess conditional risk. Since
$H(p_\phi(z_i))=H(p(x))+o_p(1)$ uniformly over the
neighborhood, we obtain
\[
\mathcal{R}_n(x;m_n^*,\alpha_n^*)
\ge H(p(x))+\kappa_c-o_p(1)
\]
for some $\kappa_c>0$.

For example, after choosing the local radius so that the Lipschitz error is
small enough, one may take $\kappa_c=c^2/16$. By Lemma
\ref{lem:cond_unif_conv}, the empirical objective is uniformly close to
$\mathcal{R}_n$, and hence
\[
\mathcal{J}_n(x;m_n^*,\alpha_n^*)
\ge
H(p(x))+\kappa_c-o_p(1).
\]

We now construct a consistent competitor. By Assumption
\ref{ass:candidate_grid}, there exists $m_n^\circ\in\mathcal{M}_n$ such that
$m_n^\circ\to\infty$ and $m_n^\circ/n\to0$, and the grid contains
$\alpha=1$. Consider the pure-local candidate $(m_n^\circ,1)$. Since
$m_n^\circ\in\mathcal{M}_n$ and
$\max\mathcal{M}_n\le m_{\sigma,n}$, its neighborhood radius satisfies
\[
d_{(m_n^\circ)}(x)\le \sigma_n.
\]

Lemma \ref{lem:bandwidth} gives $\sigma_n\to0$ in probability, while
$m_n^\circ\to\infty$ by Assumption \ref{ass:candidate_grid}. Hence the
pure-local neighborhood radius shrinks to zero and its sample size
diverges.

By the definition of the conditional local mean,
\[
\bar p_0(x,m_n^\circ)
=
\frac{
\sum_{i\in\mathcal N_{m_n^\circ}(x)}p_\phi(z_i)+u_K
}{m_n^\circ+1},
\]
where $u_K=(1/K,\ldots,1/K)$.
For a fixed test point satisfying $p(x)=p_\phi(z)$, Assumption
\ref{ass:lipschitz_data} and
$d_{(m_n^\circ)}(x)\le\sigma_n$ yield
\[
\|\bar p_0(x,m_n^\circ)-p(x)\|_2
\le
\frac{m_n^\circ}{m_n^\circ+1}L_p\sigma_n
+
\frac{\|u_K-p(x)\|_2}{m_n^\circ+1}
\xrightarrow{P}0.
\]

Furthermore, Proposition \ref{prop:good_set} gives
\[
\|\tilde p_0(x,m_n^\circ)-\bar p_0(x,m_n^\circ)\|_2
\xrightarrow{P}0.
\]
Therefore, by the triangle inequality,
\[
\|\tilde p_0(x,m_n^\circ)-p(x)\|_2
\xrightarrow{P}0.
\]
Thus both the conditional mean and the empirical local
frequency of the pure-local competitor converge to $p(x)$
in probability.

For each $i\in\mathcal N_{m_n^\circ}(x)$, the leave-one-out
predictor is conditionally independent of $Y_i$ given
$\mathcal G_n$, under the fixed-anchor assumption.
By the Lipschitz continuity of $p_\phi$ and the localization
of the neighborhood, its conditional mean satisfies
\[
\left\|
\mathbb{E}\left[
q_{i,n}^{(-i)}(m_n^\circ,1)\mid\mathcal G_n
\right]-p_\phi(z_i)
\right\|_2
\leq 2L_p\sigma_n+O(1/m_n^\circ)
=o_p(1).
\]

By conditional Hoeffding concentration and
$\log n=o(m_{\min,n})$, the leave-one-out predictions
concentrate uniformly around their conditional means.
Under Assumption \ref{ass:bounded_loss}, all their
coordinates exceed $\delta$ with probability tending
to one. The log-loss is then bounded and continuous.
The complementary event contributes $o_p(1)$ to
the conditional risk, since its probability is
exponentially small in $m_{\min,n}$ and Laplace
smoothing bounds the loss by $O(\log n)$. Therefore,
\[
\mathcal R_n(x;m_n^\circ,1)
=
\sum_{i\in\mathcal N_{m_n^\circ}(x)}
w_i^{(m_n^\circ)}(x)
H\bigl(p_\phi(z_i)\bigr)
+o_p(1).
\]

Since the neighborhood radius converges to zero and
$p_\phi$ is Lipschitz continuous, the entropy terms
converge uniformly to $H(p(x))$. Therefore,
\[
\mathcal R_n(x;m_n^\circ,1)
=
H(p(x))+o_p(1).
\]

Using Lemma \ref{lem:cond_unif_conv} again gives
\[
\mathcal{J}_n(x;m_n^\circ,1)
\le
H(p(x))+o_p(1).
\]

For sufficiently large $n$, the competitor has strictly smaller empirical
objective than any candidate with the non-vanishing bias described above:
\[
\mathcal{J}_n(x;m_n^\circ,1)
<
\mathcal{J}_n(x;m_n^*,\alpha_n^*)
\]
with probability tending to one on that event. This contradicts the definition
of $(m_n^*,\alpha_n^*)$ as a minimizer of $\mathcal{J}_n$. Hence the selected
conditional mean bias must vanish in probability:
\[
\|\bar q_n(x;m_n^*,\alpha_n^*)-p(x)\|_2\xrightarrow{P}0.
\]
Finally, since
\[
\operatorname{Bias}_{\mathrm{LAPE}}(x;m_n^*)
=
\min_{\alpha\in[0,1]}
\left\|
(1-\alpha)\pi^{\mathrm{anchor}}(x)
+
\alpha\bar p_0(x,m_n^*)
-
p(x)
\right\|_2
\]
is no larger than the same expression evaluated at $\alpha=\alpha_n^*$, the
last display also implies
\[
\operatorname{Bias}_{\mathrm{LAPE}}(x;m_n^*)\xrightarrow{P}0.
\]
\end{proof}

\textbf{Step 3: Completing the proof.}

Let $E_n(\varepsilon):=\{\|\hat p_n(x)-p(x)\|_2>\varepsilon\}.$
Using the definition of the final estimator and the conditional mean
$\bar q_n$, we decompose
\begin{align*}
\|\hat p_n(x)-p(x)\|_2
&\le
\|\hat p_n(x)-\bar q_n(x;m_n^*,\alpha_n^*)\|_2
+
\|\bar q_n(x;m_n^*,\alpha_n^*)-p(x)\|_2 \\
&\le
\sup_{m\in\mathcal{M}_n}
\|\tilde p_0(x,m)-\bar p_0(x,m)\|_2
+
\|\bar q_n(x;m_n^*,\alpha_n^*)-p(x)\|_2 .
\end{align*}
The first term converges to zero in probability by Proposition
\ref{prop:good_set}; the second term converges to zero in probability by
Proposition \ref{thm:bias_conv}. Therefore, for every fixed
$\varepsilon>0$,
\[
\mathbb{P}(\|\hat p_n(x)-p(x)\|_2>\varepsilon)\to0.
\]
This proves Theorem 5.1.

\section{Experimental Details}
\label{app:experimental_details}

\subsection{Datasets}
\label{app:datasets}

The experiments cover controlled synthetic diagnostics, tabular classification, medical-image transfer learning, natural-image classification, and multi-view or multimodal settings. The tabular suite includes Obesity Risk~\citep{palechor2019dataset}, Financial Fraud~\citep{dal2015calibrating}, Shuttle and other Statlog-style datasets~\citep{king1995statlog}, ACS Income and Adult~\citep{ding2021retiring,kohavi1996scaling}, Letter Recognition and its noisy variant~\citep{frey1991letter}, Forest Cover~\citep{blackard1999comparative}, MiniBooNE~\citep{roe2005boosted}, EEG Eye State~\citep{roesler2013eeg}, Image Segmentation~\citep{uci1990imageseg}, and Vehicle Silhouettes~\citep{mowforth_vehicle_uci}. The visual suite includes MedMNIST tasks such as PneumoniaMNIST, DermaMNIST, and BreastMNIST~\citep{medmnistv2}, MHIST~\citep{wei2021petri}, CIFAR-10/100~\citep{Krizhevsky2009LearningML}, CIFAR-10-C~\citep{hendrycks2019benchmarking}, and DTD~\citep{cimpoi2014describing}. The multi-view and multimodal evaluations include UCI Multiple Features~\citep{multiple_features_72}, the Optical Recognition of Handwritten Digits dataset~\citep{alpaydin1998optical}, and HAM10000 image--metadata tasks~\citep{tschandl2018ham10000}. Synthetic diagnostics are generated specifically to isolate local heterogeneity, calibration/reference-ratio effects, covariate shift, high-dimensional retrieval, structural noise, architecture depth, adaptive weighting, anchor-capacity and cross-fitting effects, and multi-view interaction.

\subsection{Additional Baselines}
\label{app:additional_baselines}
In addition to classical post-hoc calibrators such as TS/VS~\citep{guo2017calibration}, Dirichlet calibration~\citep{kull2019beyond}, LTS~\citep{Ding_2021_ICCV}, ETS~\citep{zhang2020mix}, and isotonic-regression variants such as IROvA/MIR~\citep{zadrozny2001obtaining,zadrozny2002transforming}, we include recent feature-aware and grouping-based baselines, including FC~\citep{tao2025fc} and SAG~\citep{yang2023sag}. When a baseline can serve as a calibrated anchor, we also evaluate the corresponding LAPE-side refinement on top of that anchor using the same reference split. Detailed full-result tables are collected in Appendix~\ref{app:supplementary_results}.

\subsection{Additional Metrics}
\label{app:additional_metrics}
The main text focuses on NLL, ECE, and accuracy for readability. Complete appendix tables also report ACE when applicable. For a test set of size $N$, let $\hat y_i=\arg\max_y\hat p_y(x_i)$ and $\hat c_i=\max_y\hat p_y(x_i)$. Accuracy is
\[
\mathrm{Acc}=\frac{1}{N}\sum_{i=1}^{N}\mathbf{1}\{\hat y_i=y_i\}.
\]
NLL and ECE are defined in Section~\ref{sec:experiments}. ACE follows the adaptive calibration-error protocol of \citet{nixon2019measuring}: for each class $k$, examples are sorted by the predicted probability $\hat p_k(x_i)$, divided into equal-frequency adaptive bins $\mathcal{B}_{k,m}$, and compared by
\[
\mathrm{ACE}=\frac{1}{K}\sum_{k=1}^{K}\sum_{m=1}^{M}\frac{|\mathcal{B}_{k,m}|}{N}
\left|
\frac{1}{|\mathcal{B}_{k,m}|}\sum_{i\in\mathcal{B}_{k,m}}\mathbf{1}\{y_i=k\}
-
\frac{1}{|\mathcal{B}_{k,m}|}\sum_{i\in\mathcal{B}_{k,m}}\hat p_k(x_i)
\right|.
\]
These metrics are diagnostic complements rather than separate headline objectives: lower NLL, ECE, and ACE are better, while higher accuracy is better. Appendix~\ref{app:supplementary_results} reports the corresponding values for the full method suite.

For appendix tables that compare a calibrated anchor with its LAPE-refined version, 
we report the relative percentage improvement rather than the raw metric value. 
For an error-type metric \(M\), such as NLL, ECE, or ACE, the reported value is
\[
\Delta M(\%) =
100 \times \frac{M(\text{anchor}) - M(\text{anchor}+\text{LAPE})}
{M(\text{anchor})}.
\]
For accuracy, the reported value is
\[
\Delta \mathrm{Acc}(\%) =
100 \times \frac{\mathrm{Acc}(\text{anchor}+\text{LAPE}) - \mathrm{Acc}(\text{anchor})}
{\mathrm{Acc}(\text{anchor})}.
\]
Thus, a positive value always indicates an improvement after applying LAPE, 
whereas a negative value indicates that the LAPE refinement slightly worsens the 
corresponding metric. We use this relative form because LAPE is designed as a 
post-hoc refinement layer on top of an existing anchor rather than as a standalone 
replacement. Reporting relative changes therefore directly measures the additional 
benefit contributed by LAPE and makes improvements comparable across datasets, 
anchors, and metrics with different numerical scales.

\subsection{Backbones and Experimental Frameworks}
\label{app:backbone_frameworks}

Table references below summarize the backbone or experimental framework used by each supplementary result table. The choice of backbone follows the role of each experiment. For image data, we use standard transfer-learning backbones such as ResNet and DenseNet. For tabular data, we use MLP predictors because the datasets are moderate in size and overly complex deep architectures can easily overfit. The same principle is used for synthetic simulations, where compact MLPs are sufficient to isolate the intended diagnostic factor. Unless otherwise stated, synthetic single-view experiments use an MLP with hidden layers $(64,32)$ and ReLU activations. The architecture-depth diagnostic separately uses shallow, medium, and deep MLPs with hidden layers $(32)$, $(64,32)$, and $(128,64,32)$. When an experiment is specially constructed to study a stress scenario rather than to benchmark a standard dataset, the detailed construction is described in the corresponding subsection. The post-hoc calibration and LAPE refinement experiments were run on a single local workstation with an Intel Core 5 120U CPU with 10 cores and 12 logical processors and 32 GB RAM; all paired comparisons within each table were conducted under the same hardware and software environment.

Table~\ref{tab:experimental_framework_summary} summarizes the experimental frameworks used in the supplementary experiments, including the dataset or diagnostic setting and the corresponding backbone or calibration framework for each result table.
\begin{table}[t]
\centering
\caption{\textbf{Summary of experimental frameworks for supplementary experiments.}
Each row reports the dataset or diagnostic setting, together with the backbone or calibration framework used for the corresponding result table.}
\label{tab:experimental_framework_summary}
\scriptsize
\setlength{\tabcolsep}{4pt}
\renewcommand{\arraystretch}{1.10}
\begin{tabularx}{\textwidth}{p{0.15\textwidth}p{0.30\textwidth}X}
\toprule
\textbf{Table} & \textbf{Dataset / setting} & \textbf{Backbone or framework} \\
\midrule

Table~\ref{tab:sim1_full}
& GMM local heterogeneity
& Special oracle-based probability-distortion design; details are given in the corresponding subsection. \\

Table~\ref{tab:sim2_full}
& Sample-ratio sensitivity
& Fixed MLP $(64,32)$ and fixed test set; only the calibration/reference ratio is varied. \\

Table~\ref{tab:sim3_full}
& Covariate shift
& Source-trained MLP $(64,32)$ with target-domain calibration/reference splits. \\

Table~\ref{tab:cifar10c_gaussian}
& CIFAR-10-C Gaussian-noise shift
& Frozen ImageNet-pretrained ResNet18 with a multinomial linear probe trained on clean CIFAR-10; target-domain labeled reference samples are used for calibration. \\

Table~\ref{tab:sim7_full}
& High-dimensional stress test
& MLP $(64,32)$ on synthetic high-dimensional features. \\

Table~\ref{tab:sim5_full_b}
& Structural-noise two moons
& MLP $(64,32)$ on two-moons data with controlled structural noise. \\

Table~\ref{tab:architecture_depth_full}
& Obesity architecture depth
& MLP $(32)$, MLP $(64,32)$, and MLP $(128,64,32)$. \\

Table~\ref{tab:sim8_full}
& Ice-and-Fire ablation
& Special zone-based diagnostic for adaptive weights; details are given in the corresponding subsection. \\

Table~\ref{tab:coop_dac_lape}
& DTD with prompt-tuned CLIP
& OpenAI CLIP ViT-B/32 with CoOp prompt tuning; DAC is additionally evaluated as a CLIP-specific calibration anchor. \\

Table~\ref{tab:train_time_lape}
& Obesity training-time calibration
& MLP $(128,64)$ trained with CE, MDCA, or MbLS objectives, followed by post-hoc LAPE refinement. \\

Table~\ref{tab:cross_fitting}
& Cross-fitting ablation
& Directly generated features with a flexible decision-tree anchor for an anchor-capacity ablation; direct reuse, 50/50 splitting, and 5-fold cross-fitting are compared. \\

Table~\ref{tab:cfaf_synthetic_full}
& Multi-view and noisy-view CFAF stress tests
& View-specific MLP predictors with CFAF retrieval in the augmented representation; the noisy-view block uses five views, with two weak informative views and three pure random views. \\

Table~\ref{tab:cfaf_digits_full}
& UCI Multiple Features
& Two descriptor views are used in this table: Fourier descriptors and Zernike descriptors. Each view is fitted by an MLP $(64,32)$ and calibrated by VS; late fusion combines the two calibrated view-wise probabilities, while CFAF-LAPE performs local refinement in the augmented two-view representation. \\

Table~\ref{tab:cfaf_ham_full}
& HAM10000 image--metadata
& Two binary tasks are reported: malignant-vs-benign and melanoma-vs-nevus. The image view uses frozen ImageNet-pretrained ResNet18 features with a linear head and VS calibration. The metadata view uses age, sex, and anatomical-location variables processed by imputation/scaling/one-hot encoding, followed by an MLP $(64,32)$ and VS calibration. CFAF-LAPE uses the concatenated image--metadata representation for multimodal local refinement. \\

Table~\ref{tab:dynamic_fusion_lape}
& IID multi-view digits
& View-specific MLP $(48,24)$ predictors with QMF or PDF dynamic fusion; CFAF-LAPE performs post-hoc local refinement in the joint representation. \\

Table~\ref{tab:appendix_tabular_full}
& Full tabular suite
& Standard tabular MLP predictors; all methods share the same backbone within each dataset. \\

Table~\ref{tab:image_transfer_full}
& Visual transfer learning
& ResNet18/DenseNet121 for medical tasks; ResNet50 plus linear probe for CIFAR. \\

\bottomrule
\end{tabularx}
\end{table}

\section{Supplementary Experimental Results}\label{app:supplementary_results}

\subsection{Computational Cost}
\label{app:complexity}

We distinguish the one-time index construction cost from per-query inference cost. Before inference, LAPE builds a spatial index such as a KD-tree~\citep{bentley1975multidimensional} or Ball-tree~\citep{omohundro1989five} on $\mathcal{D}_{\text{ref}}$. This costs $\mathcal{O}(N \log N)$ or $\mathcal{O}(N d \log N)$ depending on the index and dimension, and is amortized across future queries.

The online inference process begins with the neighbor retrieval phase, denoted as $T_{\text{retrieve}}$. This component represents the computational time required to query the spatial index and identify the $K_{\text{max}}$ nearest neighbors for a test sample. By leveraging the pre-built index, we avoid the prohibitive linear scan $\mathcal{O}(N \cdot d)$ associated with naive methods. Instead, the search complexity is reduced to $\mathcal{O}(d \log N)$ in the average case. In high-dimensional regimes where distance metrics may degrade, applying PCA as a preconditioner ensures the effective dimension $d'$ remains small, strictly maintaining the efficient logarithmic bound $T_{\text{retrieve}} \approx \mathcal{O}(d' \log N)$.

Following retrieval, the algorithm enters the parameter optimization phase, denoted as $T_{\text{optim}}$. This component measures the time needed to determine the optimal neighborhood size $k^* \in [k_{\text{min}}, K_{\text{max}}]$ and interpolation weight $\alpha^* \in \mathcal{A}$. We employ a LOO protocol that iterates through candidate sizes $k$. For a specific neighborhood size $k$, computing the LOO likelihood involves aggregating evidence from $k$ neighbors. Summing over the search space results in an arithmetic progression of operations $\sum_{k=1}^{K_{\text{max}}} k \approx \mathcal{O}(K_{\text{max}}^2)$. Consequently, considering the grid search over $\alpha$, the optimization complexity is derived as $T_{\text{optim}} = \mathcal{O}(|\mathcal{A}| \cdot K_{\text{max}}^2)$.

Combining these two components, the total inference complexity per query is given by
\begin{equation}
    T_{\text{total}} = T_{\text{retrieve}} + T_{\text{optim}} \approx \mathcal{O}(d' \log N + |\mathcal{A}| \cdot K_{\text{max}}^2),
\end{equation}
where $|\mathcal{A}|$ is the number of interpolation candidates. The quadratic term is controlled by using a moderate $K_{\max}$, motivated by the consistency analysis, and by evaluating $k$ on a grid in practice. The timing study below verifies that this implementation is practical for the dataset sizes used in the paper.

\paragraph{Empirical runtime with acceleration.}

We compare the naive implementation with the accelerated version used in our reruns. The accelerated version implements the same LAPE search in vectorized query batches, so predicted probabilities are identical to the naive loop while inference overhead is lower. We repeat the timing experiment over five random seeds with query batch size $256$.

\begin{table}[H]
\centering
\caption{\textbf{Multi-seed inference-time comparison between the naive and accelerated implementations.} Latency is reported in milliseconds per sample as mean $\pm$ standard deviation over five seeds. Lower latency is better, and higher speedup is better.}
\label{tab:appendix_runtime_speed}
\scriptsize
\setlength{\tabcolsep}{6pt}
\renewcommand{\arraystretch}{1.08}
\begin{tabular}{ccccc}
\toprule
\textbf{$N_{\mathrm{ref}}$} & \textbf{Seeds} & \textbf{Naive ms/sample} & \textbf{Accelerated ms/sample} & \textbf{Speedup} \\
\midrule
250 & 5 & $1.437 \pm 0.264$ & $0.146 \pm 0.019$ & $9.81\times \pm 1.23$ \\
500 & 5 & $2.415 \pm 0.104$ & $0.293 \pm 0.047$ & $8.38\times \pm 1.21$ \\
1000 & 5 & $2.115 \pm 0.019$ & $0.326 \pm 0.029$ & $6.53\times \pm 0.63$ \\
2000 & 5 & $3.033 \pm 0.707$ & $0.632 \pm 0.150$ & $5.00\times \pm 1.56$ \\
\bottomrule
\end{tabular}
\end{table}



\subsection{Oracle Neighborhood-Variation Diagnostic}\label{appendix:Oracle_Neighborhood}

\begin{table}[H]
\centering
\caption{\textbf{Oracle neighborhood-variation diagnostic in a fixed feature space.} Test points are ranked by oracle neighborhood variation $V_k(x)$ with fixed diagnostic size $k=80$. The table reports the bottom $30\%$ and top $30\%$ groups, averaged over five seeds. Lower local bias, MAE, and KL are better; $V_k(x)$ is a diagnostic variation score. 
The oracle neighborhood-variation score is defined below; local bias,
MAE, and KL are assessed relative to the known oracle conditional
probabilities.
}
\label{tab:oracle_smoothness_diagnostic}
\scriptsize
\setlength{\tabcolsep}{5pt}
\renewcommand{\arraystretch}{1.08}
\begin{tabular}{lcccccc}
\toprule
\textbf{Group} & $\mathbf{V_k(x)}$ & \textbf{Local bias} & \textbf{Base MAE} & \textbf{LAPE MAE} & \textbf{Base KL} & \textbf{LAPE KL} \\
\midrule
Low oracle variation & 0.0073 & 0.0024 & 0.0151 & 0.0145 & 0.0124 & 0.0110 \\
High oracle variation & 0.2056 & 0.1542 & 0.1118 & 0.0943 & 0.0727 & 0.0472 \\
\bottomrule
\end{tabular}
\end{table}

In synthetic experiments, the true conditional probability function $\eta$ is known. We therefore use it only for analysis to measure how much the true probability changes inside a fixed diagnostic neighborhood of each test point:
$V_k(x)=k^{-1}\sum_{j\in\mathcal{N}_k(x)}|\eta(x_j)-\eta(x)|$.
Here $k$ is fixed for the diagnostic split ($k=80$ in our implementation) and is separate from the adaptive $k^*$ selected by LAPE. This statistic can be viewed as a finite-neighborhood proxy for the local variation rate of the oracle probability around $x$: small $V_k(x)$ means that the retrieved neighbors have similar oracle probabilities to the query, while large $V_k(x)$ means that the retrieved neighbors already differ from the query in their oracle probabilities. We compute this score for every test point in the same retrieval feature space, sort test points by $V_k(x)$, and report the bottom $30\%$ as the low-variation group and the top $30\%$ as the high-variation group. Table~\ref{tab:oracle_smoothness_diagnostic} shows that LAPE is closer to the oracle probability than the base predictor in both groups, with a larger reduction in the high-variation group, although that group remains intrinsically harder.

\subsection{Local-Heterogeneity Stress Test}

To isolate local heterogeneity from feature-space mismatch, we keep a single GMM feature space and a fixed oracle probability function $\eta(x)$. We then simulate region-specific calibration errors by transforming the oracle probabilities into base probabilities with different temperatures and class-bias shifts across mixture components. Thus, some local neighborhoods are over-confident, others are under-confident, and some have class-specific bias, while the retrieval geometry itself remains unchanged.

Table~\ref{tab:sim1_full} reports raw metric values rather than relative LAPE-side changes. This experiment is designed to isolate local heterogeneity from feature-space mismatch and to illustrate the role of local evidence, so we compare all available methods directly instead of reporting pairwise improvement ratios. Specifically, we generate a fixed GMM feature space and a fixed oracle
conditional probability function, and then construct the base probabilities by applying region-dependent distortions to the oracle probabilities. These distortions use different temperature factors and class-bias shifts across mixture regions, so different neighborhoods require different calibration corrections while the retrieval geometry itself remains unchanged. The results show that global calibrators are limited in this heterogeneous setting because one global remapping cannot correct neighborhoods with different residual errors. LTS is relatively competitive because its local temperature map can partially respond to region-dependent temperature variation. However, the distortion also includes class-bias shifts and more general local residual patterns that are not fully captured by a temperature-only correction. LAPE-side refinements, including LTS+LAPE, further improve the calibration metrics by combining a global or local anchor with empirical label evidence from nearby calibration samples. This supports the intended mechanism: neighborhood evidence helps LAPE handle local heterogeneity that is difficult for global post-hoc calibration and not fully resolved by parametric local-temperature scaling.

\begingroup
\tiny
\setlength{\tabcolsep}{3pt}
\renewcommand{\arraystretch}{1.03}
\begin{longtable}{p{0.27\textwidth}p{0.19\textwidth}rrrr}
\caption{\textbf{Full local-heterogeneity stress-test results on the GMM manifold.} All methods are evaluated under the controlled probability-distortion setting.}
\label{tab:sim1_full}\\
\toprule
\textbf{Setting} & \textbf{Method} & \textbf{NLL} & \textbf{Acc. (\%)} & \textbf{ECE} & \textbf{ACE} \\
\midrule
\endfirsthead
\toprule
\textbf{Setting} & \textbf{Method} & \textbf{NLL} & \textbf{Acc. (\%)} & \textbf{ECE} & \textbf{ACE} \\
\midrule
\endhead
GMM local heterogeneity & Base & 1.0010 & 60.13 & 0.1084 & 0.1103 \\
   & TS & 0.8480 & 60.13 & 0.0394 & 0.0822 \\
   & ETS & 0.8429 & 60.13 & 0.0301 & 0.0789 \\
   & VS & 0.8493 & 60.40 & 0.0444 & 0.0852 \\
   & Dir & 0.8488 & 59.87 & 0.0390 & 0.0863 \\
   & LTS & 0.8171 & 60.13 & 0.0422 & 0.0605 \\
   & IROvA & 0.8356 & 64.93 & 0.0515 & 0.0491 \\
   & MIR & 0.7873 & 60.13 & 0.0228 & 0.0452 \\
   & Pure LAPE & 0.7963 & 65.07 & 0.0407 & 0.0380 \\
   & TS+LAPE & 0.7812 & 64.20 & 0.0267 & 0.0349 \\
   & ETS+LAPE & 0.7830 & 64.33 & 0.0238 & 0.0366 \\
   & VS+LAPE & 0.7816 & 64.13 & 0.0263 & 0.0344 \\
   & Dir+LAPE & 0.7818 & 64.13 & 0.0268 & 0.0351 \\
   & LTS+LAPE & 0.7832 & 64.93 & 0.0335 & 0.0324 \\
   & IROvA+LAPE & 0.8123 & 64.33 & 0.0291 & 0.0339 \\
   & MIR+LAPE & 0.7839 & 64.40 & 0.0234 & 0.0375 \\
\bottomrule
\end{longtable}
\endgroup

\subsection{Sample-Ratio Sensitivity}
\begin{table}[H]
\centering
\caption{\textbf{Sample-ratio sensitivity in the local-residual diagnostic.} Each method cell reports NLL / ECE / accuracy (\%) averaged over five calibration split seeds.}
\label{tab:sim2_full}
\scriptsize
\setlength{\tabcolsep}{4pt}
\renewcommand{\arraystretch}{1.08}
\resizebox{\textwidth}{!}{%
\begin{tabular}{lccccc}
\toprule
\textbf{Ratio} & \textbf{Base} & \textbf{VS} & \textbf{Pure LAPE} & \textbf{VS+LAPE} & \textbf{Gain over VS} \\
\midrule
10\% & 0.1205 / 0.0128 / 97.75 & 0.1116 / 0.0227 / 97.37 & 0.1284 / 0.0642 / 96.89 & 0.0869 / 0.0220 / 97.39 & 22.18\% \\
20\% & 0.1205 / 0.0128 / 97.75 & 0.0974 / 0.0165 / 97.55 & 0.1003 / 0.0429 / 97.33 & 0.0817 / 0.0193 / 97.65 & 16.14\% \\
40\% & 0.1205 / 0.0128 / 97.75 & 0.0959 / 0.0172 / 97.47 & 0.0886 / 0.0326 / 97.48 & 0.0801 / 0.0212 / 97.61 & 16.42\% \\
60\% & 0.1205 / 0.0128 / 97.75 & 0.0951 / 0.0179 / 97.59 & 0.0796 / 0.0267 / 97.61 & 0.0756 / 0.0216 / 97.74 & 20.45\% \\
80\% & 0.1205 / 0.0128 / 97.75 & 0.0954 / 0.0189 / 97.68 & 0.0759 / 0.0240 / 97.69 & 0.0739 / 0.0215 / 97.78 & 22.53\% \\
100\% & 0.1205 / 0.0128 / 97.75 & 0.0951 / 0.0185 / 97.68 & 0.0744 / 0.0229 / 97.77 & 0.0727 / 0.0202 / 97.77 & 23.52\% \\
\bottomrule
\end{tabular}
}
\end{table}

This table is the numeric counterpart of Figure~\ref{fig:sample_ratio_sensitivity}. 
It fixes the base MLP and test set within each seed, and varies only the calibration ratio used by post-hoc methods. 
Since the goal is to isolate how the amount of calibration data affects the LAPE refinement, rather than to run another exhaustive baseline comparison, we report a compact set of representative methods: the uncalibrated base model, VS, Pure LAPE, and VS+LAPE.

The gain over VS is not strictly monotone. 
With very small calibration sets, VS still leaves substantial local residual error, so even sparse neighborhood evidence can yield a visible gain.  At intermediate ratios, VS itself benefits from more calibration samples and narrows the relative gap. 
As the calibration set becomes larger, however, the fixed global form of VS reaches a ceiling, whereas the local empirical component in LAPE continues to improve, so the gain rises again at larger ratios.

\subsection{Covariate-Shift Robustness}

Table~\ref{tab:sim3_full} reports the target-calibrated covariate-shift diagnostic under a controlled GMM simulation.  
The base MLP is trained on easy/core samples from the unshifted source distribution, while calibration samples and test samples are drawn from the full target distribution at each shift level.  
This design creates a controlled mismatch between the source-trained anchor and the target population, so the experiment evaluates whether target-domain calibration samples can repair local residual errors left by the anchor. 
To focus on robustness to distributional shift rather than on a difficult base classification problem, we deliberately use a relatively friendly GMM setting.  
As a result, under the mild shift level, global score-based calibrators such as TS and VS can already remove much of the dominant logit-scale miscalibration, leaving only modest room for additional local refinement.  
As the shift becomes stronger, however, the remaining errors become more structured in the target feature space, and the value of target-domain neighborhood evidence becomes clearer.  
The LAPE-side variants therefore show larger NLL and accuracy improvements at higher shift levels, supporting the role of LAPE as a post-hoc local refinement layer under covariate shift.

\begingroup
\tiny
\setlength{\tabcolsep}{3pt}
\renewcommand{\arraystretch}{1.03}
\begin{longtable}{p{0.27\textwidth}p{0.19\textwidth}rrrr}
\caption{\textbf{Target-calibrated covariate-shift LAPE-side results.} Entries are relative percentage changes versus the corresponding anchor, with positive values indicating improvement.}
\label{tab:sim3_full}\\
\toprule
\textbf{Shift} & \textbf{Method} & \textbf{$\Delta$NLL (\%)} & \textbf{$\Delta$Acc. (\%)} & \textbf{$\Delta$ECE (\%)} & \textbf{$\Delta$ACE (\%)} \\
\midrule
\endfirsthead
\toprule
\textbf{Shift} & \textbf{Method} & \textbf{$\Delta$NLL (\%)} & \textbf{$\Delta$Acc. (\%)} & \textbf{$\Delta$ECE (\%)} & \textbf{$\Delta$ACE (\%)} \\
\midrule
\endhead

Shift 0.4 & Pure LAPE & +3.83 & +0.12 & +15.60 & +18.67 \\
  & TS+LAPE & +1.83 & +0.18 & +9.40 & -10.42 \\
  & VS+LAPE & +1.32 & +0.18 & -3.05 & -11.29 \\
  & Dir+LAPE & +1.40 & +0.18 & -8.96 & +1.10 \\
  & LTS+LAPE & +2.00 & +0.18 & +14.51 & +1.65 \\
  & ETS+LAPE & +1.85 & +0.18 & +8.10 & -8.81 \\
  & IROvA+LAPE & +7.10 & +0.24 & +6.30 & -1.02 \\
  & MIR+LAPE & +2.52 & +0.35 & -7.64 & -0.60 \\

\midrule

Shift 1.2 & Pure LAPE & +4.50 & +0.24 & +31.22 & +9.47 \\
  & TS+LAPE & +1.49 & +0.54 & +4.57 & -6.42 \\
  & VS+LAPE & +1.12 & +0.42 & +14.56 & -0.97 \\
  & Dir+LAPE & +1.69 & +0.24 & +5.82 & -2.11 \\
  & LTS+LAPE & +1.38 & +0.60 & -3.36 & -7.40 \\
  & ETS+LAPE & +1.51 & +0.54 & +4.71 & -6.40 \\
  & IROvA+LAPE & +14.12 & +0.42 & +4.79 & +0.66 \\
  & MIR+LAPE & +6.18 & +0.48 & +7.26 & -6.35 \\

\bottomrule
\end{longtable}
\endgroup

\begingroup 
\tiny 
\setlength{\tabcolsep}{2.7pt} 
\renewcommand{\arraystretch}{1.02} 
\begin{longtable}{p{0.16\textwidth}p{0.19\textwidth}rrrr} 
\caption{\textbf{CIFAR-10-C Gaussian-noise LAPE-side results.} 
Entries are relative percentage changes versus the corresponding anchor, with positive values indicating improvement. 
Severity levels 1, 3, and 5 represent increasingly strong Gaussian-noise corruption.} 
\label{tab:cifar10c_gaussian}\\ 
\toprule 
\textbf{Severity} & \textbf{Method} 
& \textbf{$\Delta$NLL (\%)} 
& \textbf{$\Delta$Acc. (\%)} 
& \textbf{$\Delta$ECE (\%)} 
& \textbf{$\Delta$ACE (\%)} 
\\ 
\midrule 
\endfirsthead 
 
\toprule 
\textbf{Severity} & \textbf{Method} 
& \textbf{$\Delta$NLL (\%)} 
& \textbf{$\Delta$Acc. (\%)} 
& \textbf{$\Delta$ECE (\%)} 
& \textbf{$\Delta$ACE (\%)} 
\\ 
\midrule 
\endhead 
 
1 & Pure LAPE & +39.18 & +16.80 & +87.12 & +49.68 \\ 
  & TS+LAPE & +17.44 & +20.10 & -109.93 & +27.95 \\ 
  & VS+LAPE & +2.61 & +2.39 & -82.98 & -40.95 \\ 
  & Dir+LAPE & +1.50 & +1.16 & -45.51 & -12.64 \\ 
  & LTS+LAPE & +16.50 & +20.70 & -97.52 & +27.26 \\ 
  & ETS+LAPE & +17.47 & +20.01 & -104.91 & +28.14 \\ 
  & IROvA+LAPE & +11.76 & +5.52 & -30.33 & +5.48 \\ 
  & MIR+LAPE & +17.83 & +20.44 & -180.20 & +30.25 \\ 
  & FC+LAPE & +23.23 & +16.71 & +78.90 & +39.37 \\ 
  & FC+TS+LAPE & +16.87 & +18.51 & -182.03 & +26.48 \\ 
  & FC+ETS+LAPE & +16.90 & +18.43 & -182.64 & +26.81 \\ 
  & SAG+LAPE & +21.86 & +20.01 & +0.68 & +32.06 \\ 
 
\midrule 
 
3 & Pure LAPE & +64.35 & +109.43 & +87.99 & +75.32 \\ 
  & TS+LAPE & +28.78 & +111.20 & -192.58 & +56.93 \\ 
  & VS+LAPE & +13.29 & +22.17 & -366.12 & -91.80 \\ 
  & Dir+LAPE & +7.66 & +11.51 & -134.34 & -48.87 \\ 
  & LTS+LAPE & +28.94 & +112.49 & -51.23 & +56.86 \\ 
  & ETS+LAPE & +29.04 & +111.31 & -137.70 & +56.49 \\ 
  & IROvA+LAPE & +20.93 & +34.50 & -78.84 & -0.48 \\ 
  & MIR+LAPE & +28.45 & +118.10 & -242.36 & +52.47 \\ 
  & FC+LAPE & +42.99 & +110.80 & +73.18 & +67.33 \\ 
  & FC+TS+LAPE & +29.19 & +112.23 & -130.01 & +54.54 \\ 
  & FC+ETS+LAPE & +29.19 & +112.23 & -130.01 & +54.54 \\ 
  & SAG+LAPE & +28.46 & +112.59 & -44.55 & +54.46 \\ 
 
\midrule 
 
5 & Pure LAPE & +67.27 & +150.13 & +89.12 & +77.53 \\ 
  & TS+LAPE & +26.93 & +152.24 & -388.26 & +55.20 \\ 
  & VS+LAPE & +15.80 & +37.61 & -339.08 & -105.90 \\ 
  & Dir+LAPE & +9.17 & +17.75 & -126.83 & -46.43 \\ 
  & LTS+LAPE & +29.01 & +152.28 & +9.74 & +59.08 \\ 
  & ETS+LAPE & +27.37 & +152.42 & -281.51 & +54.73 \\ 
  & IROvA+LAPE & +21.03 & +46.90 & -83.35 & -19.50 \\ 
  & MIR+LAPE & +26.62 & +157.83 & -355.54 & +49.38 \\ 
  & FC+LAPE & +44.76 & +142.46 & +76.86 & +70.55 \\ 
  & FC+TS+LAPE & +27.06 & +143.72 & -454.94 & +51.75 \\ 
  & FC+ETS+LAPE & +27.06 & +143.72 & -454.94 & +51.75 \\ 
  & SAG+LAPE & +26.14 & +151.59 & -5.39 & +54.92 \\ 
 
\bottomrule 
\end{longtable} 
\endgroup

We further evaluate LAPE under real-image distribution shift on CIFAR-10-C.  
We consider Gaussian noise and extend the evaluation from severity 1 to severity levels 1, 3, and 5, corresponding to increasingly strong target shifts.  
A ResNet-18 linear probe is trained on clean CIFAR-10, while target-domain samples from each corrupted condition are used for calibration and evaluation.

Table~\ref{tab:cifar10c_gaussian} reports the effect of adding LAPE across a broad range of calibration anchors.  
LAPE improves the NLL across nearly all anchor--severity combinations, while accuracy also shows broad improvements.  
Moreover, the NLL gains generally increase as the corruption severity becomes stronger.  
These results demonstrate that LAPE remains effective under real-image distribution shift and provides robust post-hoc refinement across different calibration baselines.

\subsection{High-Dimensional Stress Test}

\begingroup
\tiny
\setlength{\tabcolsep}{3pt}
\renewcommand{\arraystretch}{1.03}
\begin{longtable}{p{0.27\textwidth}p{0.19\textwidth}rrrr}
\caption{\textbf{Full high-dimensional stress-test results.} Only LAPE-side refinements are reported for each feature dimension. Entries are relative percentage changes versus the corresponding anchor, with positive values indicating improvement.}
\label{tab:sim7_full}\\
\toprule
\textbf{Dimension} & \textbf{Method} & \textbf{$\Delta$NLL (\%)} & \textbf{$\Delta$Acc. (\%)} & \textbf{$\Delta$ECE (\%)} & \textbf{$\Delta$ACE (\%)} \\
\midrule
\endfirsthead
\toprule
\textbf{Dimension} & \textbf{Method} & \textbf{$\Delta$NLL (\%)} & \textbf{$\Delta$Acc. (\%)} & \textbf{$\Delta$ECE (\%)} & \textbf{$\Delta$ACE (\%)} \\
\midrule
\endhead
D=10 & ETS+LAPE & +0.74 & +0.37 & -60.00 & -9.39 \\
  & LTS+LAPE & +2.78 & +0.37 & -3.54 & -17.19 \\
  & Dir+LAPE & +4.67 & +0.76 & -25.49 & -1.92 \\
  & TS+LAPE & +3.59 & +0.49 & -21.49 & -12.20 \\
  & VS+LAPE & +3.79 & +0.62 & -22.80 & -10.65 \\
  & IROvA+LAPE & +7.74 & +0.26 & +8.33 & -5.41 \\
  & Pure LAPE & +12.94 & +0.49 & +39.34 & +13.36 \\
  & MIR+LAPE & +0.91 & +0.37 & -8.70 & -10.14 \\
\midrule
D=50 & ETS+LAPE & +2.65 & +0.00 & -2.46 & -6.18 \\
  & TS+LAPE & +2.69 & +0.00 & -0.90 & -6.42 \\
  & Dir+LAPE & +3.44 & +0.58 & +12.74 & -2.43 \\
  & VS+LAPE & +2.89 & +0.60 & -12.18 & +3.16 \\
  & LTS+LAPE & +2.54 & +0.00 & +25.54 & -6.44 \\
  & MIR+LAPE & +3.72 & +0.58 & +12.29 & +6.25 \\
  & Pure LAPE & +34.46 & +1.18 & +68.69 & +57.55 \\
  & IROvA+LAPE & +10.68 & +1.48 & -9.48 & -15.93 \\
\midrule
D=100 & Dir+LAPE & +2.58 & +2.67 & -51.31 & +10.51 \\
  & VS+LAPE & +2.95 & +2.99 & -47.14 & +1.87 \\
  & ETS+LAPE & +2.62 & +1.99 & -45.55 & +6.11 \\
  & TS+LAPE & +2.62 & +1.99 & -45.55 & +6.11 \\
  & LTS+LAPE & +2.43 & +1.81 & -43.05 & -2.43 \\
  & MIR+LAPE & +11.44 & +1.65 & -25.86 & -5.62 \\
  & IROvA+LAPE & +11.64 & +2.18 & -41.86 & -5.73 \\
  & Pure LAPE & +34.92 & +2.16 & +66.80 & +71.75 \\
\midrule
D=200 & LTS+LAPE & +1.08 & +1.11 & -5.26 & -10.22 \\
  & Dir+LAPE & +1.33 & +1.09 & -5.17 & +13.11 \\
  & MIR+LAPE & +2.35 & +1.29 & +24.47 & -2.20 \\
  & ETS+LAPE & +2.04 & +2.04 & +18.80 & +2.81 \\
  & TS+LAPE & +1.70 & +1.65 & -10.53 & +1.04 \\
  & VS+LAPE & +1.53 & +1.09 & +15.43 & +3.36 \\
  & Pure LAPE & +38.27 & +3.51 & +73.24 & +61.77 \\
  & IROvA+LAPE & +1.85 & +1.67 & +27.90 & +8.65 \\
\bottomrule
\end{longtable}
\endgroup

\begingroup
\tiny
\setlength{\tabcolsep}{3pt}
\renewcommand{\arraystretch}{1.03}
\begin{longtable}{p{0.27\textwidth}p{0.19\textwidth}rrrr}
\caption{\textbf{Full structural-noise results on the two-moons setting.} Only LAPE-side refinements are reported at each noise level. Entries are relative percentage changes versus the corresponding anchor, with positive values indicating improvement.}
\label{tab:sim5_full_b}\\
\toprule
\textbf{Noise} & \textbf{Method} & \textbf{$\Delta$NLL (\%)} & \textbf{$\Delta$Acc. (\%)} & \textbf{$\Delta$ECE (\%)} & \textbf{$\Delta$ACE (\%)} \\
\midrule
\endfirsthead
\toprule
\textbf{Noise} & \textbf{Method} & \textbf{$\Delta$NLL (\%)} & \textbf{$\Delta$Acc. (\%)} & \textbf{$\Delta$ECE (\%)} & \textbf{$\Delta$ACE (\%)} \\
\midrule
\endhead
Noise 0.1 & TS+LAPE & +6.06 & +0.00 & +17.02 & -93.82 \\
  & ETS+LAPE & +6.06 & +0.00 & +17.02 & -93.82 \\
  & Dir+LAPE & +1.71 & +0.00 & -5.96 & +0.50 \\
  & VS+LAPE & +1.66 & +0.00 & -5.63 & +0.41 \\
  & Pure LAPE & +57.43 & +0.00 & +59.86 & +48.74 \\
  & LTS+LAPE & +3.66 & +0.00 & -11.47 & -8.32 \\
  & IROvA+LAPE & -0.61 & +0.00 & +32.86 & +7.54 \\
  & MIR+LAPE & -0.31 & +0.00 & +21.42 & -29.16 \\
\midrule
Noise 0.2 & LTS+LAPE & +13.91 & +0.31 & +26.27 & -37.63 \\
  & ETS+LAPE & +15.69 & +0.31 & +25.17 & -117.61 \\
  & TS+LAPE & +15.69 & +0.31 & +25.17 & -117.61 \\
  & VS+LAPE & +16.09 & +0.41 & +19.26 & -168.82 \\
  & Dir+LAPE & +16.09 & +0.41 & +19.25 & -168.70 \\
  & Pure LAPE & +28.41 & +0.31 & +64.40 & +67.56 \\
  & MIR+LAPE & +14.94 & +0.31 & +61.09 & +13.47 \\
  & IROvA+LAPE & +8.79 & +1.15 & +19.55 & -57.26 \\
\midrule
Noise 0.3 & LTS+LAPE & +4.41 & +1.45 & +21.80 & +27.88 \\
  & Dir+LAPE & +8.46 & +1.67 & +7.05 & +15.92 \\
  & VS+LAPE & +8.46 & +1.67 & +7.05 & +15.93 \\
  & ETS+LAPE & +7.73 & +1.56 & -6.26 & +24.01 \\
  & TS+LAPE & +7.73 & +1.56 & -6.26 & +24.01 \\
  & Pure LAPE & +9.51 & +1.45 & +47.62 & +32.75 \\
  & MIR+LAPE & +7.54 & +1.78 & -28.77 & -26.20 \\
  & IROvA+LAPE & +19.19 & +0.78 & -64.78 & -54.37 \\
\bottomrule
\end{longtable}
\endgroup

Table~\ref{tab:sim7_full} provides a deliberately challenging high-dimensional stress test for LAPE. When the feature dimension increases to 200, the relative gains become smaller than those in lower-dimensional settings, with most improvements around the one-percent level. This behavior is consistent with the difficulty of local neighborhood estimation in high-dimensional spaces. Nevertheless, LAPE still provides positive calibration gains for most anchors. This indicates that the proposed local refinement remains stable rather than degrading under high-dimensional feature representations.

Table~\ref{tab:sim5_full_b} stresses LAPE by progressively corrupting the local neighborhood information in the calibration feature space. Larger noise makes nearby calibration samples less informative for correcting the calibration residual of a test point, which naturally reduces the relative gains. Even under this stress, LAPE still yields positive calibration improvements for most anchors, indicating that the proposed local refinement can tolerate moderate degradation of local information.

\subsection{Architecture-Depth Diagnostic}
\begingroup
\tiny
\setlength{\tabcolsep}{3pt}
\renewcommand{\arraystretch}{1.03}
\begin{longtable}{p{0.27\textwidth}p{0.19\textwidth}rrrr}
\caption{\textbf{Architecture-depth results on Obesity.} Each block reports only LAPE-side refinements for the corresponding backbone. Entries are relative percentage changes versus the corresponding anchor, with positive values indicating improvement.}
\label{tab:architecture_depth_full}\\
\toprule
\textbf{Dataset / backbone} & \textbf{Method} & \textbf{$\Delta$NLL (\%)} & \textbf{$\Delta$Acc. (\%)} & \textbf{$\Delta$ECE (\%)} & \textbf{$\Delta$ACE (\%)} \\
\midrule
\endfirsthead
\toprule
\textbf{Dataset / backbone} & \textbf{Method} & \textbf{$\Delta$NLL (\%)} & \textbf{$\Delta$Acc. (\%)} & \textbf{$\Delta$ECE (\%)} & \textbf{$\Delta$ACE (\%)} \\
\midrule
\endhead
Obesity / MLP-deep & ETS+LAPE & +10.91 & +1.24 & -14.09 & -6.42 \\
  & TS+LAPE & +10.88 & +1.24 & -14.29 & -5.82 \\
  & Pure LAPE & +8.55 & +1.24 & -29.34 & -12.48 \\
  & Dir+LAPE & +10.81 & -0.41 & +7.78 & -13.06 \\
  & LTS+LAPE & +5.57 & +1.24 & +3.82 & +3.21 \\
  & MIR+LAPE & +8.81 & +1.24 & +12.10 & +10.52 \\
  & VS+LAPE & +9.25 & +1.45 & +36.90 & +20.16 \\
  & IROvA+LAPE & +30.15 & +0.62 & +23.96 & +8.97 \\
\midrule
Obesity / MLP-medium & Dir+LAPE & +2.12 & +0.20 & +6.79 & -9.93 \\
  & TS+LAPE & +0.35 & +0.20 & +18.99 & -6.54 \\
  & ETS+LAPE & +0.35 & +0.20 & +18.99 & -6.54 \\
  & Pure LAPE & +0.70 & +0.20 & +35.15 & +7.50 \\
  & VS+LAPE & +1.04 & +0.40 & -26.25 & -0.88 \\
  & LTS+LAPE & +0.57 & +0.20 & +40.36 & +11.41 \\
  & MIR+LAPE & +0.61 & +0.20 & +7.53 & +8.90 \\
  & IROvA+LAPE & +9.08 & +0.41 & +18.64 & +9.87 \\
\midrule
Obesity / MLP-shallow & TS+LAPE & +2.08 & +0.22 & +25.98 & +17.72 \\
  & ETS+LAPE & +1.48 & +0.43 & +19.38 & +12.84 \\
  & LTS+LAPE & +3.49 & +0.00 & -11.85 & +0.35 \\
  & Pure LAPE & +3.44 & +0.00 & +25.03 & +19.18 \\
  & Dir+LAPE & +3.15 & +0.67 & +5.70 & +6.25 \\
  & MIR+LAPE & +9.14 & +0.22 & -19.92 & -13.91 \\
  & VS+LAPE & +2.87 & +0.00 & +15.87 & +13.65 \\
  & IROvA+LAPE & +21.20 & +0.00 & -23.10 & +2.40 \\
\bottomrule
\end{longtable}
\endgroup

The architecture-depth diagnostic shows different regimes of the Obesity backbone. The Base accuracies are 87.50\% for MLP-shallow, 93.94\% for MLP-medium, and 91.29\% for MLP-deep. The shallow model is relatively underfitted, leaving useful but less structured local information for refinement. The medium model is the best-matched backbone, with the highest accuracy and lowest Base NLL, so the remaining room for improvement is naturally limited. The deep model shows signs of mild overfitting: its Base accuracy and NLL are worse than those of the medium model. Nevertheless, LAPE still extracts useful local information from the deep representation and improves calibration for most anchors. Overall, the results suggest that LAPE can refine probabilities across underfitted, well-fitted, and mildly overfitted backbones, although the magnitude of improvement depends on how much reliable local information remains in the learned feature space.

\subsection{Adaptive-Strategy Ablation}
\begin{table}[H]
\centering
\caption{\textbf{Ice-and-Fire ablation study.} Zone columns report region-wise NLL.}
\label{tab:sim8_full}
\scriptsize
\setlength{\tabcolsep}{5pt}
\renewcommand{\arraystretch}{1.08}
\begin{tabular}{lrrrrrr}
\toprule
\textbf{Method} & \textbf{NLL} & \textbf{Acc. (\%)} & \textbf{ECE} & \textbf{ACE} & \textbf{Zone A} & \textbf{Zone B} \\
\midrule
Adaptive LAPE & 0.0384 & 99.00 & 0.0122 & 0.0070 & 0.0303 & 0.0467 \\
Fixed (A=0.2, K=30) & 0.1660 & 94.15 & 0.0589 & 0.0593 & 0.2262 & 0.1044 \\
Base Model (Alpha=0) & 0.1756 & 90.90 & 0.0339 & 0.0306 & 0.3016 & 0.0467 \\
Fixed (A=0.8, K=30) & 0.1949 & 99.25 & 0.1540 & 0.1321 & 0.0766 & 0.3161 \\
Fixed (A=0.5, K=100) & 0.2343 & 94.20 & 0.1101 & 0.1053 & 0.2494 & 0.2188 \\
\bottomrule
\end{tabular}
\end{table}

Table~\ref{tab:sim8_full} reports an ablation stress test for the adaptive weighting strategy in LAPE. The fixed-weight variants apply the same correction coefficient to all test points, so they cannot account for spatially heterogeneous fitting quality in the feature space. We therefore divide the test samples into Zone A and Zone B, corresponding to regions where the anchor predictor is fitted to different degrees: one region has more reliable anchor probabilities, while the other contains stronger local calibration residuals and requires a different correction strength. Under this deliberately heterogeneous setting, the adaptive strategy improves over fixed coefficients not only overall but also within each zone. This indicates that the benefit of LAPE is not simply an averaging effect across regions; even inside a given region, test-point-specific adaptation provides a more effective local correction than a globally fixed weight.

\subsection{Evaluations on Modern Vision-Language Models} 
\label{app:modern_vlm} 

LAPE relies on the feature space to provide locally informative label structure.   
For CLIP, whose representations are learned through broad image--text alignment, it is therefore important to verify whether the frozen image neighborhoods remain useful after downstream prompt tuning changes the class prototypes, logits, and decision boundaries. 
 
To examine this issue, we conduct a targeted experiment on DTD using OpenAI CLIP ViT-B/32 and the standard CoOp prompt-tuning objective~\citep{Zhou2022CoOp}.   
The image and text encoders are frozen, and only shared continuous context tokens are trained.   
We use the official base-to-new split with 16-shot base-class training, a prompt length of 16, and 100 training epochs.   
LAPE is evaluated on labeled novel-class reference/test splits from one fixed CoOp checkpoint, with results averaged over five stratified splits. 
Under CoOp, prompt tuning changes the text prototypes and prediction boundaries while leaving the frozen CLIP image embeddings, and hence LAPE's retrieval neighborhoods, unchanged.   
We additionally include Distance-Aware Calibration (DAC)~\citep{Wang2024DAC} to examine whether LAPE remains complementary to a CLIP-specific calibration method. 
 
\begingroup 
\tiny 
\setlength{\tabcolsep}{3pt} 
\renewcommand{\arraystretch}{1.03} 
\begin{longtable}{p{0.24\textwidth}rrrr} 
 
\caption{\textbf{CoOp/DAC LAPE-side results on DTD.} 
Entries are relative percentage changes versus the corresponding anchor, with positive values indicating improvement.} 
\label{tab:coop_dac_lape}\\ 
 
\toprule 
\textbf{Method} 
& \textbf{$\Delta$NLL (\%)} 
& \textbf{$\Delta$Acc. (\%)} 
& \textbf{$\Delta$ECE (\%)} 
& \textbf{$\Delta$ACE (\%)} 
\\ 
\midrule 
\endfirsthead 
 
\toprule 
\textbf{Method} 
& \textbf{$\Delta$NLL (\%)} 
& \textbf{$\Delta$Acc. (\%)} 
& \textbf{$\Delta$ECE (\%)} 
& \textbf{$\Delta$ACE (\%)} 
\\ 
\midrule 
\endhead 
 
CoOp+LAPE 
& +25.87 
& +20.45 
& +67.73 
& +35.74 
\\ 
 
CoOp+DAC+LAPE 
& +13.45 
& +21.15 
& +31.34 
& +24.91 
\\ 
 
\bottomrule 
\end{longtable} 
\endgroup 
 
As shown in Table~\ref{tab:coop_dac_lape}, adding LAPE consistently improves NLL and accuracy over both CoOp and CoOp+DAC.   
LAPE reduces NLL by $25.87\%$ over CoOp and by $13.45\%$ over CoOp+DAC, while improving accuracy by more than $20\%$ in both settings.   
Overall, these results show that the frozen CLIP image neighborhoods remain useful for refining prompt-adapted predictions and that LAPE provides complementary gains beyond DAC.

\subsection{Compatibility with Training-Time Calibration Objectives}
\label{app}

We further examine whether LAPE remains effective when calibration is incorporated directly into the model-training objective.
In addition to standard cross-entropy (CE) training, we consider MDCA~\citep{Hebbalaguppe2022MDCA} and MbLS~\citep{Liu2022MbLS} as representative training-time calibration objectives.
On Obesity, the full MLP is trained under each objective over five independent training seeds, after which LAPE is applied post hoc using the same calibration protocol.

\begingroup
\tiny
\setlength{\tabcolsep}{3pt}
\renewcommand{\arraystretch}{1.03}
\begin{longtable}{p{0.24\textwidth}rrrr}

\caption{\textbf{Training-time calibration LAPE-side results on Obesity.}
Entries are relative percentage changes versus the corresponding anchor, with positive values indicating improvement.}
\label{tab:train_time_lape}\\

\toprule
\textbf{Method}
& \textbf{$\Delta$NLL (\%)}
& \textbf{$\Delta$Acc. (\%)}
& \textbf{$\Delta$ECE (\%)}
& \textbf{$\Delta$ACE (\%)}
\\
\midrule
\endfirsthead

\toprule
\textbf{Method}
& \textbf{$\Delta$NLL (\%)}
& \textbf{$\Delta$Acc. (\%)}
& \textbf{$\Delta$ECE (\%)}
& \textbf{$\Delta$ACE (\%)}
\\
\midrule
\endhead

CE+LAPE
& +3.27
& +0.70
& -8.95
& -3.66
\\

MDCA+LAPE
& +3.74
& +0.80
& -2.78
& -6.83
\\

MbLS+LAPE
& +22.54
& +2.40
& +7.50
& +10.40
\\

\bottomrule
\end{longtable}
\endgroup

As shown in Table~\ref{tab:train_time_lape}, adding LAPE further improves the primary NLL metric over CE-, MDCA-, and MbLS-trained predictors, while accuracy also improves.
The NLL improvement is particularly pronounced for MbLS, reaching $22.54\%$. These results show that LAPE can provide additional post-hoc improvement even when calibration has already been incorporated into the training objective.

\subsection{Cross-Fitting versus Data Splitting}
\label{app:cross_fitting}

We further examine the role of cross-fitting through a ablation study with a flexible tree-based anchor.
We compare direct data reuse, an independent 50/50 data split, and 5-fold cross-fitting using $200$ labeled calibration samples, with results averaged over 20 independent seeds.

\begingroup
\tiny
\setlength{\tabcolsep}{3pt}
\renewcommand{\arraystretch}{1.03}
\begin{longtable}{p{0.24\textwidth}rrrr}

\caption{\textbf{Cross-fitting versus data-splitting results.}
Results are reported as mean$\pm$standard deviation over 20 independent seeds.}
\label{tab:cross_fitting}\\

\toprule
\textbf{Strategy}
& \textbf{NLL}
& \textbf{Acc.}
& \textbf{ECE}
& \textbf{ACE}
\\
\midrule
\endfirsthead

\toprule
\textbf{Strategy}
& \textbf{NLL}
& \textbf{Acc.}
& \textbf{ECE}
& \textbf{ACE}
\\
\midrule
\endhead

Direct reuse
& 1.3884$\pm$0.1132
& 0.6485$\pm$0.0291
& 0.3316$\pm$0.0291
& 0.3227$\pm$0.0300
\\

50/50 split
& 0.6198$\pm$0.0304
& 0.6816$\pm$0.0193
& 0.0957$\pm$0.0302
& 0.0935$\pm$0.0303
\\

5-fold cross-fitting
& \textbf{0.5951$\pm$0.0284}
& \textbf{0.7021$\pm$0.0228}
& \textbf{0.0757$\pm$0.0206}
& \textbf{0.0728$\pm$0.0193}
\\

\bottomrule
\end{longtable}
\endgroup

As shown in Table~\ref{tab:cross_fitting}, 5-fold cross-fitting achieves the best performance across all four metrics.
It avoids direct in-sample anchor evaluation while retaining all labeled calibration samples, whereas the 50/50 split sacrifices half of the available data at each stage.
These results support cross-fitting as a data-efficient choice when flexible anchors are used under limited calibration data.
\subsection{CFAF Experiments}
\label{app:cfaf_experiments}

This subsection reports CFAF-related experiments separately from the tabular leaderboard.
The tables include synthetic multi-view and noisy-view stress tests, UCI multi-view digits, and HAM10000 image-plus-metadata experiments.

\begin{table}[t]
\centering
\tiny
\setlength{\tabcolsep}{3pt}
\renewcommand{\arraystretch}{1.03}
\caption{\textbf{Full CFAF synthetic stress-test results.}}
\label{tab:cfaf_synthetic_full}

\begin{tabular}{p{0.25\textwidth}p{0.25\textwidth}rrrr}
\toprule
\textbf{Scenario}
& \textbf{Method}
& \textbf{NLL}
& \textbf{Acc. (\%)}
& \textbf{ECE}
& \textbf{ACE}
\\
\midrule

Interaction (XOR)
& CFAF-LAPE
& 0.4134
& 84.80
& 0.0567
& 0.0639
\\

& Fused (VS+Weight)
& 0.5067
& 78.96
& 0.0985
& 0.1037
\\

& VS View A
& 0.5439
& 68.72
& 0.0708
& 0.0753
\\

\midrule

Local Heterogeneity (Blobs)
& CFAF-LAPE
& 0.3940
& 79.36
& 0.0696
& 0.0696
\\

& Fused (VS+Weight)
& 0.4643
& 78.88
& 0.0473
& 0.0610
\\

& VS View B
& 0.4651
& 78.96
& 0.0424
& 0.0600
\\

\midrule

Noisy XOR views
& CFAF (ZCA + LAPE)
& 0.6102
& 69.47
& 0.0665
& 0.0661
\\

& Optimized Fusion
& 0.6813
& 64.60
& 0.1209
& 0.1536
\\

& Best Single (View 0)
& 0.6847
& 54.67
& 0.0123
& 0.0461
\\

\bottomrule
\end{tabular}
\end{table}
\textbf{Cross-view interaction.}
The interaction setting reflects the main motivation of CFAF: the predictive signal may not be fully visible from either view alone, but emerges from their combination. In this case, view-specific calibration is insufficient because each single-view predictor only observes a partial and potentially ambiguous signal. CFAF addresses this limitation by constructing the local correction in an augmented multi-view representation, allowing the calibration neighborhood to capture reliability patterns that depend jointly on the two views.

\textbf{Nonlinear fusion beyond late fusion.}
The nonlinear stress test examines a different limitation of simple late fusion. A fixed linear combination of calibrated view-wise probabilities can improve over single-view predictors, but it cannot represent nonlinear cross-view structures when the reliability pattern varies across the feature space. CFAF provides a post-hoc alternative: instead of learning a new end-to-end fusion network, it uses the augmented feature representation to perform local residual correction, thereby recovering part of the nonlinear fusion effect through calibration-based neighborhood information.

\textbf{Noisy multi-view signals.}
The noisy-view stress test evaluates whether CFAF remains useful when the joint representation contains both informative and irrelevant views. In this setting, we construct five views: the first two views contain weak but useful XOR-related information, while the remaining three views are purely random. This setting should not be interpreted as assuming that arbitrary noise views are helpful. Rather, it tests whether CFAF can still extract the informative component of the joint representation when the multi-view space is contaminated by irrelevant views and single-view predictors become unreliable. The results suggest that CFAF can tolerate such noisy multi-view representations and still recover useful cross-view calibration information.

\paragraph{Real-data results.}
This table summarizes the real-data CFAF cases. On the multi-view UCI Multiple Features benchmark, CFAF-LAPE improves over both single-view predictors and late fusion. On the multimodal HAM10000 image-plus-metadata tasks, CFAF-LAPE also achieves favorable NLL and calibration results compared with single-modality and late-fusion baselines. These results suggest that the augmented local-refinement strategy is useful not only for controlled multi-view stress tests, but also for real multi-view and multimodal data.

\begingroup
\tiny
\setlength{\tabcolsep}{3pt}
\renewcommand{\arraystretch}{1.03}
\begin{longtable}{p{0.25\textwidth}p{0.24\textwidth}rrrr}
\caption{\textbf{Real multi-view and multimodal CFAF results.} Each block reports two single-view baselines, a late-fusion baseline, and CFAF-LAPE.}
\label{tab:cfaf_digits_full}\label{tab:cfaf_ham_full}\\
\toprule
\textbf{Dataset} & \textbf{Method} & \textbf{NLL} & \textbf{Acc. (\%)} & \textbf{ECE} & \textbf{ACE} \\
\midrule
\endfirsthead
\toprule
\textbf{Dataset} & \textbf{Method} & \textbf{NLL} & \textbf{Acc. (\%)} & \textbf{ECE} & \textbf{ACE} \\
\midrule
\endhead
UCI Multiple Features & View A only & 0.1654 & 94.00 & 0.0259 & 0.0373 \\
  & View B only & 0.1497 & 95.20 & 0.0235 & 0.0194 \\
  & Late fusion & 0.1117 & 96.00 & 0.0294 & 0.0147 \\
  & CFAF-LAPE & 0.0919 & 96.40 & 0.0252 & 0.0156 \\
\midrule
HAM10000 malignant vs benign & Image only & 0.3461 & 84.98 & 0.0146 & 0.0250 \\
  & Metadata only & 0.4167 & 81.95 & 0.0210 & 0.0259 \\
  & Late fusion & 0.3409 & 84.85 & 0.0223 & 0.0215 \\
  & CFAF-LAPE & 0.3194 & 85.71 & 0.0099 & 0.0130 \\
\midrule
HAM10000 melanoma vs nevus & Image only & 0.2551 & 89.28 & 0.0199 & 0.0155 \\
  & Metadata only & 0.3748 & 86.03 & 0.0358 & 0.0704 \\
  & Late fusion & 0.2553 & 89.56 & 0.0234 & 0.0242 \\
  & CFAF-LAPE & 0.2243 & 90.68 & 0.0140 & 0.0209 \\
\bottomrule
\end{longtable}
\endgroup

\subsection{Compatibility with Stronger Dynamic-Fusion Anchors}
\label{app:dynamic_fusion}
We further evaluate whether CFAF-LAPE remains useful for stronger sample-adaptive multi-view fusion strategies.
We include QMF~\citep{Zhang2023QMF} and PDF~\citep{Cao2024PDF} as representative learned dynamic-fusion methods.
QMF performs sample-adaptive quality-aware fusion, while PDF further incorporates confidence and cross-view collaborative information into the fusion process.

Experiments are conducted on a two-view handwritten-digit benchmark, with results averaged over ten independent seeds.
For each comparison, CFAF-LAPE is applied only as a post-hoc local residual refinement on top of the trained fusion anchor, without modifying the underlying fusion model.

\begingroup
\tiny
\setlength{\tabcolsep}{3pt}
\renewcommand{\arraystretch}{1.03}
\begin{longtable}{p{0.24\textwidth}rrrr}

\caption{\textbf{Dynamic-fusion LAPE-side results on the multi-view digits benchmark.}
Entries are relative percentage changes versus the corresponding anchor, with positive values indicating improvement.}
\label{tab:dynamic_fusion_lape}\\

\toprule
\textbf{Method}
& \textbf{$\Delta$NLL (\%)}
& \textbf{$\Delta$Acc. (\%)}
& \textbf{$\Delta$ECE (\%)}
& \textbf{$\Delta$ACE (\%)}
\\
\midrule
\endfirsthead

\toprule
\textbf{Method}
& \textbf{$\Delta$NLL (\%)}
& \textbf{$\Delta$Acc. (\%)}
& \textbf{$\Delta$ECE (\%)}
& \textbf{$\Delta$ACE (\%)}
\\
\midrule
\endhead

QMF+CFAF-LAPE
& +44.67
& +2.81
& +55.93
& +56.95
\\

PDF+CFAF-LAPE
& +25.49
& +2.13
& +22.54
& +32.22
\\

\bottomrule
\end{longtable}
\endgroup

As shown in Table~\ref{tab:dynamic_fusion_lape}, adding CFAF-LAPE improves the primary NLL metric and accuracy over both QMF and PDF.
The relative NLL reductions are $44.67\%$ over QMF and $25.49\%$ over PDF, while accuracy improves by $2.81\%$ and $2.13\%$, respectively.
The mean ECE and ACE are also improved for both anchors.
These results show that CFAF-LAPE remains complementary to learned sample-adaptive multi-view fusion by providing additional post-hoc refinement through local evidence in the joint representation.}
\subsection{Additional Tabular and Visual Benchmarks}
\label{app:additional_tabular_visual}
\paragraph{Supplementary results on diverse tabular datasets.}
Table~\ref{tab:appendix_tabular_full} reports the full tabular results across
multiple application domains. Since these entries are relative percentage
changes against the corresponding anchors, they should be interpreted as the
additional refinement contributed by LAPE rather than as a standalone
leaderboard among different anchors. Except for the financial-fraud dataset,
where the gains are limited, LAPE improves NLL for most anchor--dataset pairs.
These results support the practical utility of LAPE as a post-hoc refinement
layer across heterogeneous tabular problems.

\paragraph{Visual transfer-learning results across medical and natural-image benchmarks.}
Table~\ref{tab:image_transfer_full} reports the combined visual transfer-learning
results. The medical experiments evaluate LAPE on feature spaces induced by
different ImageNet-pretrained backbones, including ResNet18 and DenseNet121,
while the CIFAR experiments use frozen ImageNet-pretrained ResNet50 features
with a multinomial linear probe. These results suggest that LAPE is not specific
to a single visual backbone: it can serve as a plug-in refinement layer for
different visual classifiers and calibrators without retraining the underlying
visual model.

Taken together, the tabular and visual results show a common pattern: LAPE most
consistently improves NLL, the primary probabilistic objective in our evaluation,
while largely preserving accuracy across datasets and backbones. Diagnostic
calibration metrics such as ECE also improve in many cases, and remain
comparable when the anchor is already strong or when the local retrieval geometry
is less informative. This behavior is consistent with the role of LAPE as a
residual probability-refinement module: it improves probabilistic quality without
substantially changing the decision rule of the underlying classifier.

\begin{table}[p]
\centering
\caption{\textbf{Full tabular LAPE-side refinement results.} Each dataset block reports relative percentage changes versus the corresponding anchor, with positive values indicating improvement; thus, $+$ is better and $-$ is worse for all $\Delta$ columns.}
\label{tab:appendix_tabular_full}
\tiny
\setlength{\tabcolsep}{2.0pt}
\renewcommand{\arraystretch}{0.78}
\begin{adjustbox}{max totalsize={\textwidth}{0.88\textheight},center}
\begin{tabular}{p{0.25\textwidth}p{0.18\textwidth}rrrr}
\toprule
\textbf{Dataset} & \textbf{Method} & \textbf{$\Delta$NLL (\%)} & \textbf{$\Delta$Acc. (\%)} & \textbf{$\Delta$ECE (\%)} & \textbf{$\Delta$ACE (\%)} \\
\midrule
Acs income & TS+LAPE & +5.15 & +1.66 & +62.71 & +48.18 \\
  & VS+LAPE & +4.65 & +1.44 & +62.05 & +29.11 \\
  & Dir+LAPE & +4.65 & +1.44 & +62.22 & +29.11 \\
  & MIR+LAPE & +4.80 & +1.61 & -21.50 & +24.35 \\
  & ETS+LAPE & +4.84 & +1.53 & -19.55 & +26.83 \\
  & IROvA+LAPE & +4.75 & +1.33 & -245.12 & +17.82 \\
  & LTS+LAPE & +2.94 & +1.45 & +18.50 & +27.58 \\
  & Pure LAPE & +11.52 & +1.45 & +43.38 & +42.98 \\
\midrule
Adult income & ETS+LAPE & +7.50 & +1.95 & -69.39 & +46.57 \\
  & VS+LAPE & +6.97 & +2.46 & +29.49 & +55.00 \\
  & Dir+LAPE & +6.97 & +2.46 & +29.48 & +55.00 \\
  & TS+LAPE & +9.25 & +2.10 & +48.28 & +62.08 \\
  & MIR+LAPE & +7.55 & +1.98 & -16.50 & +48.87 \\
  & IROvA+LAPE & +5.56 & +1.66 & -119.58 & -52.74 \\
  & LTS+LAPE & +8.19 & +1.72 & +2.24 & +43.29 \\
  & Pure LAPE & +32.77 & +1.61 & +54.58 & +54.80 \\
\midrule
Financial fraud & ETS+LAPE & -1.01 & -0.01 & -80.77 & -37.22 \\
  & TS+LAPE & -1.01 & -0.01 & -70.56 & -37.73 \\
  & VS+LAPE & -1.00 & -0.02 & -84.75 & -23.11 \\
  & Dir+LAPE & -1.00 & -0.02 & -85.90 & -23.11 \\
  & Pure LAPE & -1.03 & -0.01 & -56.84 & -46.25 \\
  & MIR+LAPE & -0.89 & -0.03 & -55.08 & -9.97 \\
  & IROvA+LAPE & -0.45 & -0.01 & -36.77 & -17.65 \\
  & LTS+LAPE & +8.24 & -0.03 & -10.66 & -5.41 \\
\midrule
Forest cover & MIR+LAPE & +3.68 & +0.58 & +8.22 & +28.92 \\
  & LTS+LAPE & +3.58 & +0.62 & -7.09 & +23.24 \\
  & ETS+LAPE & +3.63 & +0.59 & -2.74 & +28.18 \\
  & VS+LAPE & +2.66 & +0.37 & -1.53 & -4.89 \\
  & Pure LAPE & +18.94 & +0.51 & +66.93 & +60.07 \\
  & TS+LAPE & +4.10 & +0.62 & -7.49 & +17.78 \\
  & Dir+LAPE & +2.31 & +0.20 & -0.37 & -4.68 \\
  & IROvA+LAPE & +3.28 & +0.41 & +22.77 & -0.05 \\
\midrule
Letter recognition & TS+LAPE & +2.51 & +0.11 & -19.98 & -0.90 \\
  & ETS+LAPE & +2.47 & +0.11 & -28.77 & -1.59 \\
  & VS+LAPE & +1.83 & +0.16 & -36.80 & +0.93 \\
  & MIR+LAPE & +4.57 & +0.13 & -13.50 & -5.80 \\
  & Dir+LAPE & +11.21 & +0.11 & +42.91 & -1.63 \\
  & LTS+LAPE & +1.64 & +0.11 & -21.88 & -0.84 \\
  & Pure LAPE & +22.47 & +0.08 & +34.75 & +6.48 \\
  & IROvA+LAPE & +6.16 & +0.11 & -9.74 & -2.06 \\
\midrule
Miniboone & Dir+LAPE & +1.24 & +0.05 & +10.84 & +16.61 \\
  & VS+LAPE & +1.24 & +0.05 & +11.44 & +16.62 \\
  & ETS+LAPE & +1.80 & +0.10 & -10.20 & +18.55 \\
  & TS+LAPE & +1.90 & +0.11 & +6.35 & +18.88 \\
  & MIR+LAPE & +2.15 & +0.09 & +21.91 & +20.17 \\
  & IROvA+LAPE & +1.41 & +0.10 & -11.03 & -36.37 \\
  & LTS+LAPE & +2.92 & +0.09 & +24.55 & +13.28 \\
  & Pure LAPE & +6.48 & +0.09 & +29.86 & +26.89 \\
\midrule
Noisy letter & Dir+LAPE & +4.61 & +0.38 & -15.93 & -8.88 \\
  & VS+LAPE & +10.88 & +1.12 & -42.87 & -15.77 \\
  & TS+LAPE & +12.25 & +1.39 & -39.57 & -1.95 \\
  & ETS+LAPE & +12.27 & +1.42 & -31.52 & -0.92 \\
  & LTS+LAPE & +8.21 & +1.19 & -8.95 & -8.78 \\
  & MIR+LAPE & +13.96 & +1.48 & -10.97 & -2.41 \\
  & IROvA+LAPE & +22.81 & +1.57 & -25.69 & -4.75 \\
  & Pure LAPE & +29.77 & +2.02 & +45.78 & +47.99 \\
\midrule
Obesity risk & Dir+LAPE & +7.03 & +0.00 & +19.77 & -13.53 \\
  & MIR+LAPE & +21.59 & +0.00 & -17.03 & +18.50 \\
  & ETS+LAPE & +3.02 & +0.00 & -20.53 & -9.68 \\
  & TS+LAPE & +3.14 & +0.00 & -2.24 & +1.02 \\
  & Pure LAPE & +7.32 & +0.00 & +9.32 & +11.34 \\
  & IROvA+LAPE & +40.57 & +0.00 & +4.84 & +1.89 \\
  & LTS+LAPE & +0.90 & +0.00 & -4.89 & +9.31 \\
  & VS+LAPE & +9.11 & +0.00 & +9.43 & +5.98 \\
\midrule
Shuttle & Dir+LAPE & +52.91 & +0.00 & -35.62 & -79.43 \\
  & IROvA+LAPE & +53.89 & +0.06 & +34.46 & -15.00 \\
  & ETS+LAPE & +31.09 & +0.06 & +32.52 & +3.02 \\
  & MIR+LAPE & +27.75 & +0.03 & -19.75 & +16.27 \\
  & TS+LAPE & +45.93 & +0.06 & +45.71 & +5.27 \\
  & Pure LAPE & +45.99 & +0.06 & +35.12 & +9.73 \\
  & VS+LAPE & +9.13 & +0.03 & +4.43 & -21.74 \\
  & LTS+LAPE & +10.21 & +0.03 & -0.35 & +4.89 \\
\bottomrule
\end{tabular}
\end{adjustbox}
\end{table}
\clearpage

\begin{table}[p]
\centering
\caption{\textbf{Combined visual transfer-learning LAPE-side results.} Medical rows use ImageNet-pretrained ResNet18 or DenseNet121 feature spaces, and CIFAR rows use frozen ImageNet-pretrained ResNet50 features with a multinomial linear probe. Entries are relative percentage changes versus the corresponding anchor, with positive values indicating improvement; thus, $+$ is better and $-$ is worse for all $\Delta$ columns.}
\label{tab:image_transfer_full}
\tiny
\setlength{\tabcolsep}{1.8pt}
\renewcommand{\arraystretch}{0.72}
\begin{adjustbox}{max totalsize={\textwidth}{0.88\textheight},center}
\begin{tabular}{p{0.30\textwidth}p{0.16\textwidth}rrrr}
\toprule
\textbf{Dataset / backbone / retrieval} & \textbf{Method} & \textbf{$\Delta$NLL (\%)} & \textbf{$\Delta$Acc. (\%)} & \textbf{$\Delta$ECE (\%)} & \textbf{$\Delta$ACE (\%)} \\
\midrule
PneumoniaMNIST / ResNet18 / raw-euclidean & Pure LAPE & +0.91 & +0.00 & +0.49 & -0.26 \\
  & TS+LAPE & -0.06 & +0.00 & -1.91 & -0.89 \\
  & VS+LAPE & +0.08 & +0.00 & +1.56 & -0.79 \\
  & Dir+LAPE & +0.01 & +0.00 & +3.35 & -0.82 \\
  & LTS+LAPE & -0.22 & +0.00 & +0.75 & +0.21 \\
  & ETS+LAPE & -0.06 & +0.00 & -1.91 & -0.89 \\
  & IROvA+LAPE & -0.08 & +0.00 & +3.55 & +0.02 \\
  & MIR+LAPE & -0.13 & -0.18 & -1.24 & +0.10 \\
  & FC+LAPE & +0.92 & +0.00 & +0.44 & -0.17 \\
  & FC+TS+LAPE & -0.07 & +0.00 & -1.72 & -0.86 \\
  & FC+ETS+LAPE & -0.07 & +0.00 & -1.72 & -0.86 \\
  & SAG+LAPE & +0.20 & +0.00 & +3.66 & +0.31 \\
\midrule
DermaMNIST / ResNet18 / std-cosine & Pure LAPE & +1.48 & +0.13 & -53.95 & +4.36 \\
  & TS+LAPE & +1.00 & +0.00 & -0.31 & +3.56 \\
  & VS+LAPE & +0.75 & -0.13 & +8.14 & -9.13 \\
  & Dir+LAPE & +1.20 & +0.13 & -40.68 & +2.89 \\
  & LTS+LAPE & +1.20 & +0.13 & +1.72 & +2.76 \\
  & ETS+LAPE & +0.96 & +0.06 & +17.13 & +4.41 \\
  & IROvA+LAPE & +4.33 & +0.32 & -12.84 & -4.04 \\
  & MIR+LAPE & +1.23 & -0.13 & +0.48 & -17.61 \\
  & FC+LAPE & +0.98 & -0.32 & -13.44 & +1.05 \\
  & FC+TS+LAPE & +0.38 & -0.45 & +11.38 & +2.68 \\
  & FC+ETS+LAPE & +0.38 & -0.38 & -5.33 & +5.41 \\
  & SAG+LAPE & +2.00 & +0.51 & -6.34 & +15.64 \\
\midrule
BreastMNIST / ResNet18 / raw-cosine & Pure LAPE & +30.11 & +0.00 & +27.62 & +5.78 \\
  & TS+LAPE & +12.56 & +0.00 & +15.49 & +12.50 \\
  & VS+LAPE & +17.35 & +0.00 & +39.56 & +6.10 \\
  & Dir+LAPE & +12.83 & +0.00 & +4.01 & +13.36 \\
  & LTS+LAPE & +12.76 & +0.00 & +19.59 & +16.81 \\
  & ETS+LAPE & +6.61 & +0.00 & +11.29 & +11.80 \\
  & IROvA+LAPE & +59.20 & +0.00 & +24.88 & -6.47 \\
  & MIR+LAPE & +50.20 & +0.00 & +30.97 & +0.11 \\
  & FC+LAPE & +18.53 & +0.73 & +23.50 & +9.07 \\
  & FC+TS+LAPE & +12.81 & +0.73 & +32.73 & +12.98 \\
  & FC+ETS+LAPE & +9.13 & +0.73 & +30.07 & +9.25 \\
  & SAG+LAPE & +67.24 & +0.00 & +33.98 & +20.16 \\
\midrule
MHIST / DenseNet121 / std-PCA64-cosine & Pure LAPE & +2.86 & +0.19 & +5.75 & +11.61 \\
  & TS+LAPE & +2.41 & +0.19 & -1.63 & -4.44 \\
  & VS+LAPE & +2.53 & -0.19 & +9.06 & -0.71 \\
  & Dir+LAPE & +1.53 & +0.76 & +8.59 & -6.54 \\
  & LTS+LAPE & +1.36 & +0.19 & -6.70 & -10.01 \\
  & ETS+LAPE & +2.41 & +0.19 & -1.63 & -4.44 \\
  & IROvA+LAPE & +8.84 & +0.19 & +13.03 & +0.86 \\
  & MIR+LAPE & +9.04 & +0.00 & +5.14 & +0.12 \\
  & FC+LAPE & +2.86 & -0.56 & +32.57 & +29.83 \\
  & FC+TS+LAPE & +2.28 & -0.19 & +18.38 & +29.97 \\
  & FC+ETS+LAPE & +2.28 & -0.19 & +18.38 & +29.97 \\
  & SAG+LAPE & +5.16 & -0.19 & +27.94 & +18.33 \\
\midrule
CIFAR-10 / ResNet50 / raw & Pure LAPE & +36.37 & +0.05 & +30.41 & +9.48 \\
  & TS+LAPE & +0.24 & +0.02 & -1.45 & +0.57 \\
  & VS+LAPE & +0.19 & +0.03 & +4.14 & +2.94 \\
  & Dir+LAPE & +0.19 & +0.03 & +0.04 & +1.06 \\
  & LTS+LAPE & +0.55 & +0.00 & +6.02 & +3.65 \\
  & ETS+LAPE & +0.22 & +0.02 & +1.92 & +0.52 \\
  & IROvA+LAPE & +0.81 & +0.03 & +0.89 & +1.79 \\
  & MIR+LAPE & +0.17 & +0.01 & -2.68 & +0.27 \\
  & FC+LAPE & +18.22 & +0.16 & +29.53 & +21.39 \\
  & FC+TS+LAPE & +0.98 & +0.10 & -8.21 & +11.92 \\
  & FC+ETS+LAPE & +0.95 & +0.10 & -12.89 & +11.34 \\
  & SAG+LAPE & +2.75 & +0.05 & +20.79 & +4.65 \\
\midrule
CIFAR-100 / ResNet50 / PCA128 & Pure LAPE & +19.08 & +0.01 & +27.28 & +0.72 \\
  & TS+LAPE & +0.03 & +0.01 & -0.60 & -0.10 \\
  & VS+LAPE & +0.00 & +0.00 & +2.51 & +0.41 \\
  & Dir+LAPE & +19.16 & -0.01 & +27.25 & -3.14 \\
  & LTS+LAPE & +0.13 & +0.01 & +3.81 & +0.28 \\
  & ETS+LAPE & +0.02 & +0.01 & +3.69 & -0.11 \\
  & IROvA+LAPE & +5.13 & +0.03 & +29.93 & +1.11 \\
  & MIR+LAPE & +0.19 & +0.01 & +4.39 & +0.49 \\
  & FC+LAPE & +9.84 & +0.10 & +23.63 & +1.81 \\
  & FC+TS+LAPE & +0.14 & +0.03 & -1.59 & +0.49 \\
  & FC+ETS+LAPE & +0.14 & +0.03 & -1.59 & +0.49 \\
  & SAG+LAPE & +0.47 & +0.01 & +11.35 & +2.96 \\
\bottomrule
\end{tabular}
\end{adjustbox}
\end{table}
\clearpage

\newpage
\section*{NeurIPS Paper Checklist}

\begin{enumerate}

\item {\bf Claims}
    \item[] Question: Do the main claims made in the abstract and introduction accurately reflect the paper's contributions and scope?
    \item[] Answer: \answerYes{} 
    \item[] Justification: The claims made in the abstract and introduction accurately match theoretical and experimental results.

\item {\bf Limitations}
    \item[] Question: Does the paper discuss the limitations of the work performed by the authors?
    \item[] Answer: \answerYes{} 
    \item[] Justification: Discussion on limitations are clearly stated in the main paper.

\item {\bf Theory assumptions and proofs}
    \item[] Question: For each theoretical result, does the paper provide the full set of assumptions and a complete (and correct) proof?
    \item[] Answer: \answerYes{} 
    \item[] Justification: The assumptions and proofs of Theorem~\ref{thm:consistency_lape} are clearly presented in Appendix \ref{proofs}. 

    \item {\bf Experimental result reproducibility}
    \item[] Question: Does the paper fully disclose all the information needed to reproduce the main experimental results of the paper to the extent that it affects the main claims and/or conclusions of the paper (regardless of whether the code and data are provided or not)?
    \item[] Answer:  \answerYes{} 
    \item[] Justification: Please see Section \ref{sec:experiments}, Appendix \ref{app:experimental_details} and \ref{app:supplementary_results} for details of reproducibility. Reproducible codes are also submitted as \texttt{Supplementary Material} and will be made public on GitHub after the paper is received.

\item {\bf Open access to data and code}
    \item[] Question: Does the paper provide open access to the data and code, with sufficient instructions to faithfully reproduce the main experimental results, as described in supplemental material?
    \item[] Answer: \answerYes{} 
    \item[] Justification: Data access is stated in Appendix~\ref{app:experimental_details}. Reproducible codes are also submitted as \texttt{Supplementary Material} and will be made public on GitHub after the paper is received.

\item {\bf Experimental setting/details}
    \item[] Question: Does the paper specify all the training and test details (e.g., data splits, hyperparameters, how they were chosen, type of optimizer) necessary to understand the results?
    \item[] Answer: \answerYes{} 
    \item[] Justification:  See Section \ref{sec:experiments} and Appendix \ref{app:experimental_details} for experimental settings.

\item {\bf Experiment statistical significance}
    \item[] Question: Does the paper report error bars suitably and correctly defined or other appropriate information about the statistical significance of the experiments?
    \item[] Answer: \answerYes{} 
    \item[] Justification: All non deterministic statistics are reported as the average performance under different random seeds.

\item {\bf Experiments compute resources}
    \item[] Question: For each experiment, does the paper provide sufficient information on the computer resources (type of compute workers, memory, time of execution) needed to reproduce the experiments?
    \item[] Answer: \answerYes{} 
    \item[] Justification: Please see details in Appendix \ref{app:experimental_details}.
    
\item {\bf Code of ethics}
    \item[] Question: Does the research conducted in the paper conform, in every respect, with the NeurIPS Code of Ethics \url{https://neurips.cc/public/EthicsGuidelines}?
    \item[] Answer: \answerYes{} 
    \item[] Justification: The research conducted in the paper conforms with the NeurIPS Code of Ethics.

\item {\bf Broader impacts}
    \item[] Question: Does the paper discuss both potential positive societal impacts and negative societal impacts of the work performed?
    \item[] Answer: \answerYes{} 
    \item[] Justification: We discuss about the societal impacts of post-hoc confidence calibration in the introduction. We also emphasized the high value of the plug-and-play framework proposed in this article in enhancing the performance of existing post-hoc calibration methods.

    
\item {\bf Safeguards}
    \item[] Question: Does the paper describe safeguards that have been put in place for responsible release of data or models that have a high risk for misuse (e.g., pre-trained language models, image generators, or scraped datasets)?
    \item[] Answer: \answerNA{} 
    \item[] Justification: The paper poses no such risks.

\item {\bf Licenses for existing assets}
    \item[] Question: Are the creators or original owners of assets (e.g., code, data, models), used in the paper, properly credited and are the license and terms of use explicitly mentioned and properly respected?
    \item[] Answer: \answerNA{} 
    \item[] Justification: We have cited all the original papers that produced the code or dataset.

\item {\bf New assets}
    \item[] Question: Are new assets introduced in the paper well documented and is the documentation provided alongside the assets?
    \item[] Answer:\answerYes{} 
    \item[] Justification: We upload our code (as the new asset) in the supplement material.

\item {\bf Crowdsourcing and research with human subjects}
    \item[] Question: For crowdsourcing experiments and research with human subjects, does the paper include the full text of instructions given to participants and screenshots, if applicable, as well as details about compensation (if any)? 
    \item[] Answer: \answerNA{} 
    \item[] Justification: The paper does not involve crowdsourcing nor research with human subjects.

\item {\bf Institutional review board (IRB) approvals or equivalent for research with human subjects}
    \item[] Question: Does the paper describe potential risks incurred by study participants, whether such risks were disclosed to the subjects, and whether Institutional Review Board (IRB) approvals (or an equivalent approval/review based on the requirements of your country or institution) were obtained?
    \item[] Answer: \answerNA{} 
    \item[] Justification: The paper does not involve crowdsourcing nor research with human subjects.

\item {\bf Declaration of LLM usage}
    \item[] Question: Does the paper describe the usage of LLMs if it is an important, original, or non-standard component of the core methods in this research? Note that if the LLM is used only for writing, editing, or formatting purposes and does \emph{not} impact the core methodology, scientific rigor, or originality of the research, declaration is not required.
    \item[] Answer: \answerNA{} 
    \item[] Justification: The core method development in this research does not involve LLMs as any
important, original, or non-standard components.

\end{enumerate}

\end{document}